\documentclass[12pt]{article}

\usepackage[page,header]{appendix}
\usepackage{titletoc}

\usepackage{todonotes}
\usepackage{graphicx}
\usepackage{etex}
\usepackage{enumitem}
\setlist{parsep = -0em, itemsep = 0.25em}
\usepackage[utf8]{inputenc}
\usepackage[dvips,letterpaper,margin=1in]{geometry}
\usepackage{amsmath}
\usepackage{amssymb}
\usepackage{amsthm}
\usepackage{bm}
\usepackage{colonequals}
\usepackage{comment}
\usepackage{graphicx}
\usepackage{subcaption}
\usepackage{xcolor}
\usepackage{tikz} 
\usetikzlibrary{shadows} 
\usepackage{authblk}
\usepackage[utf8]{inputenc} 
\usepackage[T1]{fontenc}    
\usepackage{url}            
\usepackage{booktabs}       
\usepackage{amsfonts}       
\usepackage{nicefrac}       
\usepackage{microtype}      
\usepackage{hyperref}
\hypersetup{colorlinks = true,
            linkcolor = blue,
            urlcolor  = teal,
            citecolor = magenta,
            anchorcolor = blue}

\usepackage{amsmath,amsthm,amssymb}
\usepackage{algorithm,algpseudocode}
\usepackage{wrapfig}
\usepackage{esint}

\usepackage[style=numeric, sortcites=true, maxnames=4]{biblatex}
\usepackage{amsmath,amsthm,amssymb,bbm,mathtools}
\usepackage{xcolor}
\usepackage{subcaption}

\newcommand{\N}{\mathbb{N}}

\newcommand{\A}{\mathcal{A}}
\renewcommand{\AA}{\mathcal{A}_{0}}
\newcommand{\diff}{*}

\renewcommand{\P}{\mathcal{P}}
\newcommand{\D}{D}

\renewcommand{\P}{\mathcal{P}}
\newcommand{\syml}{\text{Sym}_L}
\newcommand{\symlt}{\text{Sym}_L^2}

\newcommand{\comp}{\eta}
\newcommand{\btheta}{\boldsymbol{\theta}}

\newcommand{\V}{V}
\renewcommand{\S}{{S^1}}

\newcommand{\ind}{\mathbbm{1}}
\newcommand{\tp}{p}
\newcommand{\bp}{\bold{p}}

\newcommand{\tbp}{\bold{p}}

\newcommand*\conj[1]{\overline{#1}}

\newtheorem{theorem}{Theorem}[section]
\newtheorem{lemma}[theorem]{Lemma}

\newtheorem{definition}[theorem]{Definition}
\newtheorem{assumption}[theorem]{Assumption}
\newtheorem{remark}{Remark}

\title{Exponential Separations in Symmetric Neural Networks}

\author[a]{Aaron Zweig}
\author[a,b]{Joan Bruna}

\affil[a]{Courant Institute of Mathematical Sciences, New York
  University, New York}
\affil[b]{Center for Data Science, New York University}

\begin{document}

\maketitle

\begin{abstract}
  In this work we demonstrate a novel separation between symmetric neural network architectures.  Specifically, we consider the Relational Network~\parencite{santoro2017simple} architecture as a natural generalization of the DeepSets~\parencite{zaheer2017deep} architecture, and study their representational gap. Under the restriction to analytic activation functions, we construct a symmetric function acting on sets of size $N$ with elements in dimension $D$, which can be efficiently approximated by the former architecture, but provably requires width exponential in $N$ and $D$ for the latter. 
\end{abstract}

\section{Introduction}

The modern success of deep learning can in part be attributed to architectures that enforce appropriate invariance.  Invariance to permutation of the input, i.e. treating the input as an unordered set, is a desirable property when learning \emph{symmetric} functions in such fields as particle physics and population statistics.  The simplest architectures that enforce permutation invariance treat each set element individually without allowing for interaction, as captured by the popular \emph{DeepSet} model ~\parencite{zaheer2017deep,qi2017pointnet}.

Several architectures explicitly enable interaction between set elements, the simplest being the Relational Networks~\parencite{santoro2017simple} that encode pairwise interaction.  This may be understood as an instance of \emph{self-attention}, the mechanism underlying Transformers~\parencite{vaswani2017attention}, which have emerged as powerful generic neural network architectures to process a wide variety of data, from image patches to text to physical data.  Specifically, Set Transformers \parencite{lee2019set} are special instantiations of Transformers, made permutation equivariant by omitting positional encoding of inputs, and using self-attention for pooling.  

Both the DeepSets and Relational Networks architectures are universal approximators for the class of symmetric functions.  But empirical evidence suggests an inherent advantage of symmetric networks using self-attention in synthetic settings ~\parencite{murphy2018janossy}, on point cloud data~\parencite{lee2019set} and in quantum chemistry~\parencite{pfau2020ab}.  In this work, we formalize this question in terms of approximation power, and explicitly construct symmetric functions which provably require exponentially-many neurons in the DeepSets model, yet are efficiently approximated with self-attention.

This exponential separation bears notable differences from typical separation results.  In particular, while the expressive power of a vanilla neural network is characterized by depth and width, expressiveness of symmetric networks is controlled particularly by \emph{symmetric width}.  In contrast to depth separations of vanilla neural networks~\parencite{eldan2016power}, in this work we observe width separations, where the weaker architectures (even with arbitrary depth) require exponential symmetric width to match the expressive power of stronger architectures.

\paragraph{Summary of Contributions}  In this work: 
\begin{itemize}
    \item We demonstrate a \emph{width separation} between the DeepSets and Relational Network architectures, where the former requires symmetric width $L \gg poly(N, D)$ to approximate a family of analytic symmetric functions, while the latter can approximate with polynomial efficiency.  This also answers an open question of high-dimensional DeepSets representation posed in~\textcite{wagstaff2022universal}
    \item We introduce an extension of the Hall inner product to high dimensions that preserves low-degree orthogonality of multisymmetric powersum polynomials, which may be of independent interest.
\end{itemize}

\section{Setup and Main Result}\label{sec:setup}

\subsection{Symmetric Architectures}

\begin{figure*}[ht]
\centering

\begin{subfigure}[t]{.5\textwidth}
  \centering
  \includegraphics[width=.95\linewidth]{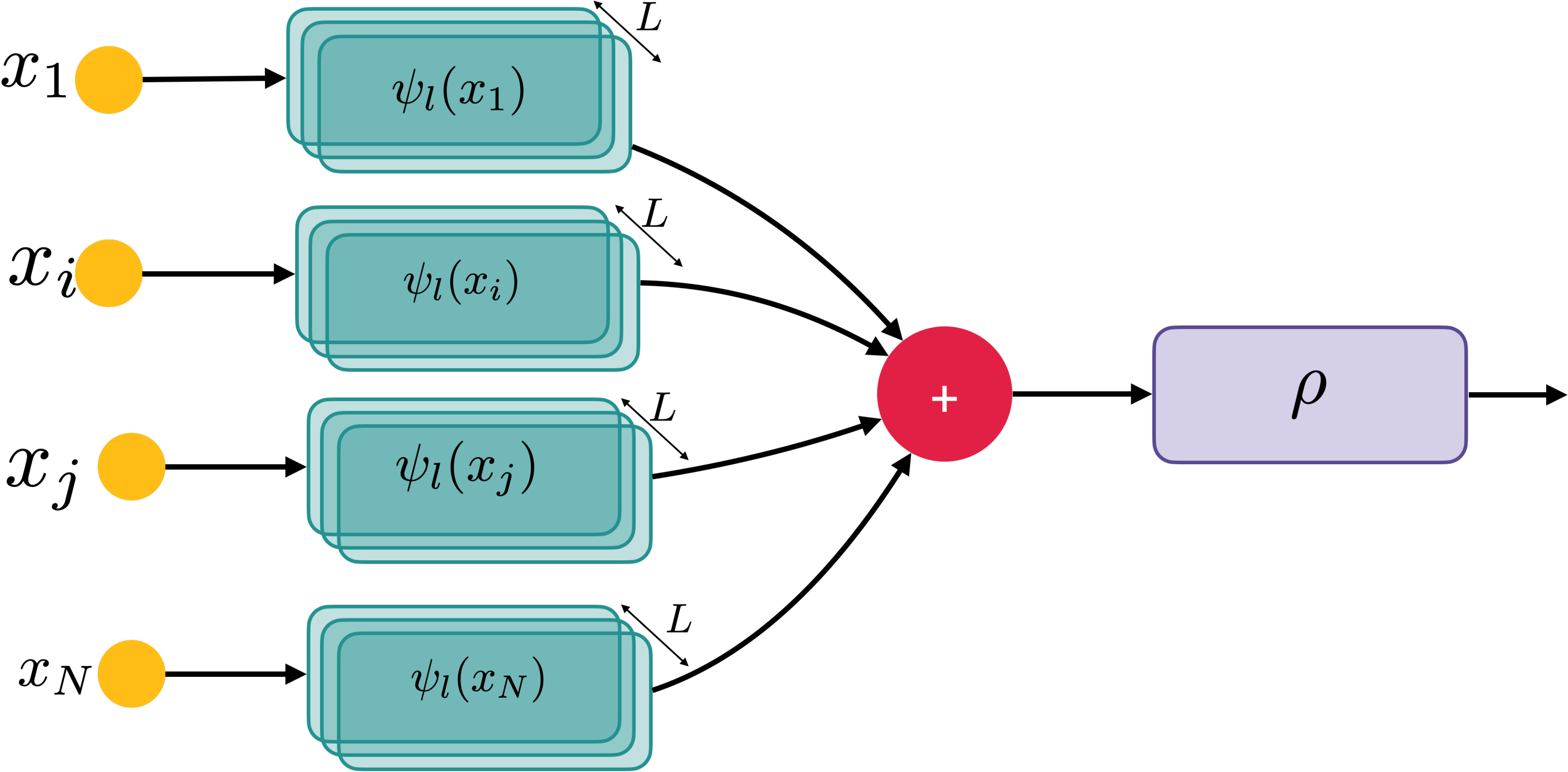}
  \caption{DeepSets with symmetric width $L$}
\end{subfigure}%
\begin{subfigure}[t]{.5\textwidth}
  \centering
  \includegraphics[width=.95\linewidth]{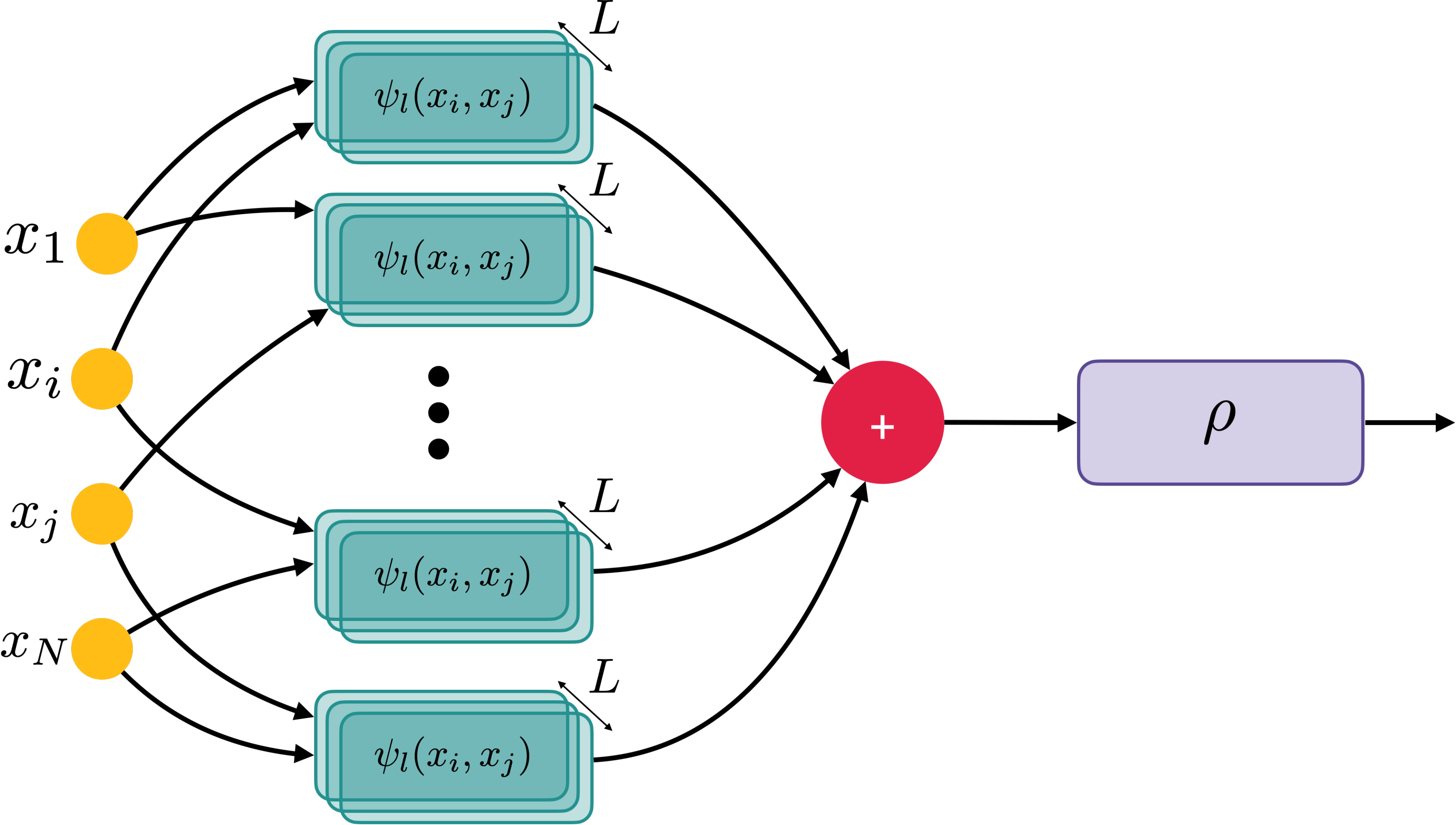}
  \caption{Relational Network with symmetric width $L$}
\end{subfigure}%

\caption{Architectural diagram for $\syml$ (left) and $\symlt$ (right)}
\label{fig:diagram}
\end{figure*}

To introduce the symmetric architectures, we must first characterize how to treat sets as inputs.  We will consider sets of size $N$, where each element of the set is a vector of dimension $D$.  In particular, we will represent a set as a matrix $X \in \mathbb{C}^{D \times N}$.  Thus, each column vector $x_n \in \mathbb{C}^D$ is an element of the  set.  Note that we consider complex-valued inputs because the natural inner product over symmetric polynomials integrates over the complex unit circle, see \textcite{macdonald1998symmetric} or Theorem~\ref{thm:hall-inner-product}.

A function $f: \mathbb{C}^{D \times N} \rightarrow \mathbb{C}$ is \emph{symmetric} if $f(X) = f(X\Pi)$ for any permutation matrix $\Pi \in  \mathbb{R}^{N\times N}$, i.e. if $f$ is invariant to permuting the columns of $X$. In other words, a symmetric function treats the input $X$ as an unordered set of column vectors.  Given the \emph{symmetric width} parameter $L$, we consider two primary symmetric architectures:
\begin{definition}
    Let $\syml$ denote the class of \emph{singleton symmetric networks} with symmetric width $L$, i.e. functions $f$ of the form:
\begin{align}\label{eq:symnn}
    f(X) & = \rho(\phi_1(X), \dots, \phi_L(X)) \\
    \phi_l(X) & = \sum_{n=1}^N \psi_l(x_n)
\end{align}
where $\{\psi_l: \mathbb{C}^D \rightarrow \mathbb{C}\}_{l=1}^L$ and $\rho: \mathbb{C}^L \rightarrow \mathbb{C}$ are arbitrary neural networks with analytic activations.

\end{definition}

The class $\syml$ is exactly the architecture of DeepSets~\parencite{zaheer2017deep} restricted to analytic activations.  However, we introduce this notation to differentiate this class from the more expressive architectures that allow for pairwise interaction among set elements.

From the theory of symmetric polynomials, if $L \geq L^* := \binom{N+D}{N} - 1$, then $f \in \syml$ is a universal approximator for any analytic symmetric function~\parencite{rydh2007minimal}.  Therefore we will primarily be interested in the expressive power of $\syml$ for $L < L^*$.

\begin{definition}
    Let $\symlt$ denote the class of \emph{pairwise symmetric networks} with symmetric width $L$, i.e. functions $f$ of the form:
\begin{align}\label{eq:psymnn}
    f(X) & = \rho(\phi_1(X), \dots, \phi_L(X)) \\
    \phi_l(X) & = \sum_{n,n'=1}^N\psi_l(x_n, x_{n'})
\end{align}
where $\{\psi_l: \mathbb{C}^{D \times D} \rightarrow \mathbb{C}\}_{l=1}^L$ and $\rho: \mathbb{C}^L \rightarrow \mathbb{C}$ are arbitrary neural networks with analytic activations.

\end{definition}

Similarly, the class $\symlt$ is exactly the architecture of Relational Pooling~\parencite{santoro2017simple} with analytic activations.  We note this architecture is also equivalent to the $2$-ary instantiation of Janossy Pooling~\parencite{murphy2018janossy}.

\subsection{Main Result}

Our main result demonstrates an exponential separation, where $\syml$ requires exponentially large symmetric width $L$ to match the expressive power of the class $\symlt$ for $L = 1$.  We choose norms to make this separation as prominent as possible: there is a hard function that can be approximated in $\symlt$ in the infinity norm, but cannot be approximated in $\syml$ even in an appropriately chosen $L_2$ norm with respect to some non-trivial data distribution.

We require one activation assumption to realize the $\symlt$ approximation:

\begin{assumption}\label{ass:act}
    The activation $\sigma : \mathbb{C} \rightarrow \mathbb{C}$ is analytic, and for a fixed $D, N$ there exist two-layer neural networks $f_1, f_2$ using $\sigma$, both with $O\left(D^2 + D \log \frac{D}{\epsilon}\right)$ width and $O(D \log D)$ bounded weights, such that:
        \begin{align}
        \sup_{|\xi| \leq 3} |f_1(\xi) - \xi^2| \leq \epsilon, \qquad
        \sup_{|\xi| \leq 3} \left|f_2(\xi) - \left(1 - (\xi/4)^{\min(D, \sqrt{N}/2)}\right) \frac{\xi - 1/4}{\xi/4 - 1} \right| \leq \epsilon
    \end{align}
\end{assumption}

Essentially this assumption guarantees that networks built with the analytic activation $\sigma$ are able to efficiently approximate the map $\xi \rightarrow \xi^2$, and, a truncated form of the finite Blaschke product\parencite{garnett2007bounded} with one zero at $\xi = 4$.  We show in Lemma~\ref{lem:expnet-epsilon} that the $\exp$ activation satisfies this assumption. 

\begin{theorem}[Exponential width-separation]\label{thm:main-result-body}
    Fix $N$ and $D > 1$, and a non-trivial data distribution $\mu$ on $D \times N$ copies of the unit complex circle $(\S)^{D \times N}$.
    
    Then there exists an analytic symmetric function $g: \mathbb{C}^{D \times N} \rightarrow \mathbb{C}$ such that $\|g\|_{L_2(\mu)} = 1$ and:
    \begin{itemize}
        \item For $L \leq N^{-2} \exp(O(\min(D, \sqrt{N}))$,
        \begin{align}
            \min_{f \in \syml} \|f - g\|_{L_2(\mu)}^2 \geq \frac{1}{12} ~.
        \end{align}
        \item There exists $f \in \symlt$ with $L = 1$, parameterized with an activation $\sigma$ that satisfies Assumption~\ref{ass:act}, with width $poly(N,D,1/\epsilon)$, depth $O(\log D)$, and max weight $O(D \log D)$ such that over $(\S)^{D\times N}$:
        \begin{align}
            \|f - g\|_\infty \leq \epsilon
        \end{align}

    \end{itemize}
\end{theorem}

\begin{remark}
    The lower bound is completely independent of the width and depth of the parameterized networks $\{\psi_l\}$ and $\rho$.  The only parameter that the theorem restricts is the symmetric width $L$.  This is in sharp contrast to the separations of vanilla networks~\parencite{eldan2016power}, where there is a natural trade-off between width and depth.
\end{remark}

\begin{remark}
    In the upper bound, we consider the network $f \in \symlt$ to have width and depth in the usual sense of vanilla neural networks, where the parameterized maps $\{\psi_l\}$ and $\rho$ obey the width, depth, and weight bounds given.
\end{remark}
\section{Related Work}

\subsection{Depth Separation}

Numerous works have studied the difference in expressive power between different neural network architectures.  Many of these works center on the representational gap between two-layer and three-layer networks~\parencite{eldan2016power,daniely2017depth}.  In particular, recent works have focused on generalizing the family of functions that realize these separations, to various radial functions~\parencite{safran2017depth} and non-radial functions~\parencite{venturi2021depth}.

A separate line of work considers separations between networks when the depth varies polynomially~\parencite{telgarsky2016benefits}.  Notably, \textcite{vardi2022width} demonstrates that depth has a greater impact on expressivity than width, in the case of vanilla neural networks.

\subsection{Symmetric Architectures}

We primarily consider the symmetric neural network parameterization as introduced in DeepSets\parencite{zaheer2017deep}, with PointNet\parencite{qi2017pointnet} a similar symmetric parameterization using a different pooling function.  Simple linear equivariant layers were also introduced in~\textcite{zaheer2017deep}.

In the context of relationships between objects in an image, the first symmetric architecture enabling explicit pairwise interaction was introduced in~\textcite{santoro2017simple}.  More complicated symmetric architectures, allowing for higher-order interaction and more substantial equivariant layers, were built on top of attention primitives~\parencite{ma2018attend,lee2019set}.  And the notion of explicit high-order interactions between set elements before symmetrizing is formalized in the architecture of Janossy pooling~\parencite{murphy2018janossy}.  

Symmetric architectures are generalized by graph neural networks~\parencite{kipf2016semi,scarselli2008graph}, under the restriction to the complete graph.  

\subsection{Symmetric Network Expressivity}

The dependence of representational power on the symmetric width parameter $L$ was first demonstrated in the $D=1$ case. Under the strong condition $L < N$, it was proven there are symmetric functions which cannot be exactly represented by a DeepSets network~\parencite{wagstaff2019limitations}, and this was later strengthened to functions which cannot be approximated in the infinity norm to arbitrary precision~\parencite{wagstaff2022universal}.

The work introducing Janossy pooling~\parencite{murphy2018janossy} also includes a theoretical result showing singleton networks cannot exactly represent some particular pairwise symmetric network.  Crucially however, this result is restricted to a simplified, non-universal symmetric architecture excluding the $\rho$ transformation, and therefore does not characterize the real-world architectures given above.

The question of expressiveness in symmetric networks may also be generalized to graph neural networks, with a focus on distinguishing non-isomorphic graphs as compared to the Weissfeler-Lehman test\parencite{xu2018powerful} and calculating invariants such as substructure counting\parencite{chen2020can}.  In particular, one may understand expressiveness in symmetric networks incorporating pairwise interaction as the ability to learn functions of the complete graph decorated with edge features.

\subsection{Symmetric Polynomial Theory}

Our proofs rely on the technical machinery of symmetric polynomial theory, thoroughly characterized in~\textcite{macdonald1998symmetric}.  In particular, we utilize the integral representation of the finite-variable Hall Inner product as introduced in Section~\ref{sec:prelim}.  Because this integral is defined over the complex unit circle, we consequently consider complex-valued neural networks~\parencite{bassey2021survey}.

The connection of symmetric networks to the powersum polynomials was first observed in~\textcite{zaheer2017deep}, and likewise the multisymmetric powersum polynomials have been applied in higher dimensional symmetric problems~\parencite{maron2019provably,segol2019universal}.  The algebraic properties of the multisymmetric powersum polynomials are well-studied, for example as a basis of higher dimensional symmetric polynomials~\parencite{rydh2007minimal} and through their algebraic dependencies~\parencite{domokos2007vector}.  However, to the best of our knowledge this is the first work to apply the Hall inner product to symmetric neural networks, and to extend this inner product to yield low-degree orthogonality over the multisymmetric polynomials.

\section{Warmup: One-dimensional set elements}

To begin, we consider the simpler case where $D = 1$, i.e. where we learn a symmetric function acting on a set of scalars.  It was already observed in~\textcite{zaheer2017deep} that the universality of DeepSets could be demonstrated by approximating the network with symmetric polynomials.  We first demonstrate that through this approximation, we can relate the symmetric width $L$ to expressive power.

\subsection{Symmetric Polynomials}

In order to approximate symmetric networks by symmetric polynomials, we choose a suitable basis.  The powersum polynomials serve as the natural choice, as their structure matches that of a singleton symmetric network, and they obey very nice orthogonality properties that we detail below.

\begin{definition}
    For $k \in \N$ and $x \in \mathbb{C}^N$, the \emph{normalized powersum polynomial} is defined as $$p_k(x) = \frac{1}{\sqrt{k}} \sum_{n=1}^N x_n^k$$ with $p_0(x) = 1$.
\end{definition}

A classical result in symmetric polynomial theory is the existence of an $L_2$ inner product that grants orthogonality for products of powersums.  To make this notion explicit and keep track of products, we index products with partitions.

\begin{definition}
    An \emph{integer partition} $\lambda$ is non-increasing, finite sequence of positive integers $\lambda_1 \geq \lambda_2 \geq \dots \geq \lambda_k$.  The weight of the partition is given by $|\lambda| = \sum_{i=1}^k \lambda_i$.  The length of a partition $l(\lambda)$ is the number of terms in the sequence.
\end{definition}

Then we characterize a product of powersums by:
\begin{align}
    p_\lambda(x) = \prod_i p_{\lambda_i}(x)
\end{align}
This notation intentionally also allows for the empty partition, such that if $\lambda = \varnothing$ then $p_\lambda = 1$.
All together, we can now state the following remarkable fact:

\begin{theorem}[{\cite[Chapter VI (9.10)]{macdonald1998symmetric}} ]\label{thm:hall-inner-product}
    There exists a $L_2(d\nu)$ inner product (for some probability measure $\nu$) such that, for partitions $\lambda, \mu$ with $|\lambda| \leq N$:
    \begin{align}
        \langle p_\lambda, p_\mu \rangle_\V = z_\lambda \ind_{\lambda = \mu}
    \end{align}
    where $z_\lambda$ is some combinatorial constant.
\end{theorem}

We index this inner product with $V$ because it is written as an expectation with respect to a density proportional to the squared Vandermonde polynomial (see Section~\ref{sec:prelim} for the precise definition).  This inner product may also be considered the finite-variable specialization of the Hall inner product, defined on symmetric polynomials over infinitely many variables~\cite[Chapter I (4.5)]{macdonald1998symmetric}. 

It's easy to check that the degree of $p_\lambda$ is equal to $|\lambda|$.  So this theorem states that the powersum terms $p_\lambda$ are "almost" an orthogonal basis, except for correlation between two high-degree terms.

Let us remark that we assume analytic activations for the sake of this theorem, as the orthogonality property does not hold for symmetric polynomials with negative exponents.  However, in exchange for that assumption we can apply this very powerful inner product, that ultimately results in the irrelevance of network depth.

\subsection{Projection Lemma}

Before we can proceed to prove a representational lower bound, we need one tool to better understand $f\in\syml$.  Utilizing the orthogonality properties of the inner product $\langle \cdot, \cdot \rangle_\V$ allows us to project any $f \in \syml$ to a simplified form, while keeping a straightforward dependence on $L$.

For example, consider some uniformly convergent power series (with no constant term) $\phi(x) = \sum_{i=1}^\infty c_{ik} p_k(x)$.  We claim $\langle p_2 p_1, \phi^3 \rangle_V = 0$.  Indeed, expanding $\phi^3$, one exclusively gets terms of the form $p_{k_1} p_{k_2} p_{k_3}$, and because the partition $\{k_1, k_2, k_3\}$ is of a different length than $\{2, 1\}$, they are clearly distinct partitions so by orthogonality $\langle p_2p_1, p_{k_1} p_{k_2} p_{k_3}\rangle_\V = 0$.

Motivated by this observation, we can project $f$ to only contain products of two terms.  Let us introduce $\P_1$ to be the orthogonal projection onto $span(\{p_t : 1 \leq t \leq N/2\})$, and $\P_2$ to be the orthogonal projection onto $span(\{p_tp_{t'} : 1 \leq t,t' \leq N/2\})$.

\begin{lemma}\label{lem:proj-one-dim}
    Given any $f \in \syml$, we may choose coefficients $v_{ij}$ over $i \leq j \leq L$, and symmetric polynomials $\phi_i$ over $i \leq L$, such that:
    \begin{align}
        \P_2 f = \sum_{i\leq j}^L v_{ij} (\P_1 \phi_i) (\P_1\phi_j)
    \end{align}
\end{lemma}

\subsection{Rank Lemma}

Given the reduced form of $f$ above, we may now go about lower bounding its approximation error to a given function $g$.

By the properties of orthogonal projection, we have $\|f-g\|_V^2 \geq \|\P_2(f - g)\|_V^2$.  And by Parseval's theorem, the function approximation error $\|\P_2 f-\P_2 g\|_V^2$ equals 
$$\sum_{t\leq t'} \left(\left\langle \P_2 f, \frac{p_{t}p_{t'}}{\|p_{t}p_{t'}\|_V} \right\rangle_V - \left\langle \P_2 g, \frac{p_{t}p_{t'}}{\|p_{t}p_{t'}\|_V} \right\rangle_V\right)^2~.$$ Rearranging the orthogonal coefficients in the form of matrices, we have the following fact:


\begin{lemma}\label{lem:rank-one-dim}
    Given any $f \in \syml$, and $g$ such that $P_2 g = g$, we have the bound
        \begin{align}
        \|\P_2f - \P_2g\|_V^2 \geq \frac{1}{2} \|F - G\|_F^2
    \end{align}
        where $F, G \in \mathbb{C}^{N/2 \times N/2}$ are matrices with entries $F_{tt'} = \left\langle \P_2f, p_t p_{t'} \right\rangle_V$, $G_{tt'} = \left\langle \P_2g, p_t p_{t'} \right\rangle_V$.  Furthermore, $F$ has maximum rank $L$.
\end{lemma}

The significance of this lemma is the rank constraint: it implies that choosing symmetric width $L$ corresponds to a maximum rank $L$ on the matrix $F$.  From here, we can use standard arguments about low-rank approximation in the Frobenius norm to yield a lower bound.

\subsection{Separation in one-dimensional case}

Our main goal in this section is to construct a hard symmetric function $g$ that cannot be efficiently approximated by $\syml$ for $L \leq N/4$.  It is not particularly expensive for the symmetric width $L$ to scale linearly with the set size $N$: however, we will use the same proof structure to prove Theorem~\ref{thm:main-result-body}, which will require $L$ to scale exponentially.

\begin{theorem}

    For $D = 1$:
        \begin{align}
        \max_{\|g\|_\V = 1} \min_{f\in\syml} \|f - g\|_\V^2 \geq 1 - \frac{2L}{N}
    \end{align}
        In particular, for $L = \frac{N}{4}$ we recover a constant lower bound of $\frac{1}{2}$.
\end{theorem}
\begin{proof}[Proof (sketch).]
    Choose $g$ such that $\P_2 g = g$.  Then because $\P_2$ is an orthogonal projection and applying Lemma~\ref{lem:rank-one-dim}:
    \begin{align}
        \min_{f\in\syml} \|f - g\|_\V^2 & \geq \min_{f\in\syml} \|\P_2 f - \P_2 g \|_\V^2 \\
        & \geq \frac{1}{2} \min_{\text{rank}(F) \leq L} \|F - G\|_F^2
    \end{align}
    
    We note that $\|p_tp_t\|_V^2 = z_{\{t,t\}} = 2$, so the choice of $g = \frac{1}{\sqrt{N}} \sum_{t=1}^{N/2}  \tp_t \tp_t$ can be seen to obey $\|g\|_V = 1$, and implies that $G$ is the scaled identity matrix $\frac{2}{\sqrt{N}} I \in \mathbb{C}^{N/2 \times N/2}$.  Then by standard properties of the SVD:
        \begin{align}
        \min_{f\in\syml} \|f - g\|_\V^2 &
        \geq \frac{1}{2}  \min_{\text{rank}(F) \leq L} \|F - \frac{2}{\sqrt{N}}I\|_F^2 \\
        & = \frac{1}{N/2} \min_{\text{rank}(F) \leq L} \|F - I\|_F^2 \\
        & = \frac{1}{N/2} (N/2 - L) \\
        & = 1 - \frac{2L}{N}
    \end{align}
    \end{proof}
\section{Proof Sketch of Main Result}\label{sec:sketch}

\subsection{Challenges for High-dimensional Set Elements}

We'd like to strengthen this separation in several ways:
\begin{itemize}
    \item Generalize to the $D > 1$ case,
    \item Realize a separation where the symmetric width $L$ must scale exponentially in $N$ and $D$, showing that $\syml$ is infeasible,
    \item Show the hard function $g$ can nevertheless be efficiently approximated in $\symlt$ for $L$ polynomial in $N$ and $D$
\end{itemize}

First, in order to approximate via polynomials in the high-dimenionsal case, we will require the high-dimensional analogue of powersum polynomials:

\begin{definition}
For a multi-index $\alpha \in \N^{D}$, the \emph{normalized multisymmetric powersum polynomial} is defined as:
\begin{align}
    \bp_{\alpha}(X) & = \frac{1}{\sqrt{|\alpha|}} \sum_n \prod_d x_{dn}^{\alpha_d}~.
\end{align}
\end{definition}

So the plan is to find a high-dimensional analogue of Lemma~\ref{lem:proj-one-dim} and Lemma~\ref{lem:rank-one-dim}, now using multisymmetric powersum polynomials, mimic the proof of the $D = 1$ case, and then additionally show the hard function $g$ is efficiently computable in the pairwise symmetric architecture.  Note that because the algebraic basis of multisymmetric powersum polynomials is of size $L^* = \binom{N + D}{N} - 1$, we can expect an exponential separation when we apply a similar rank argument.\footnote{We subtract one in order to discount the constant polynomial.}

\subsection{Sketch of Main result (lower bound)}\label{sec:lower-bound}

Because we are in high dimensions, we cannot simply apply the restricted Hall inner product introduced in Theorem~\ref{thm:hall-inner-product}.  To the best of our knowledge, there is no standard generalization of the Hall inner product to multi-symmetric polynomials that preserves the orthogonality property.  For the main technical ingredient in the high-dimensional case we introduce a novel generalization, which builds on two inner products.

First, we introduce a new input distribution $\nu$ over set inputs $X \in \mathbb{C}^{D \times N}$, and induce an $L_2$ inner product:
\begin{align}
    \langle f, g \rangle_\A = \mathbb{E}_{X\sim \nu}\left[f(X) \overline{g(X)}\right]~.
\end{align}
We use this inner product to measure the approximation error of $\syml$.  That is, we seek a lower bound to $\min_{f \in \syml} \|f - g\|_\A$, for a suitable choice of hard function $g$.

We can now apply an analogue of Lemma~\ref{lem:proj-one-dim} to project $f$ to a simplified form.  But we cannot immediately apply an analogue of Lemma~\ref{lem:rank-one-dim}, as it relied on Parseval's theorem and the low-degree multisymmetric powersum polynomials are not orthogonal in this inner product.  Put another way, if we represent $\langle \cdot, \cdot \rangle_\A$ as a matrix in the basis of low-degree multisymmetric powersums, it will be positive-definite but include some off-diagonal terms.

The idea is to now introduce a new inner product with a different input distribution $\nu_0$
\begin{align}
    \langle f, g \rangle_{\AA} = \mathbb{E}_{X\sim \nu_0}\left[f(X) \overline{g(X)}\right]~,
\end{align}
and define the bilinear form
\begin{align}
    \langle f, g \rangle_{\diff} = \langle f, g \rangle_{\A}  - 2 \langle f, g \rangle_{\AA} ~.
\end{align}

Typically positive-definiteness is lost when subtracting two inner products, but we prove that $\langle \cdot, \cdot \rangle_\diff$ is an inner product when restricted to a particular subspace of symmetric polynomials (see Theorem~\ref{thm:bilinear-form}).  Furthermore, the careful choice of $\nu$ and $\nu_0$ cancels the off-diagonal correlation of different multisymmetric powersums, so they are orthogonal under this new inner product $\langle \cdot, \cdot \rangle_\diff$.

By the norm domination $\|\cdot\|_\A \geq \|\cdot\|_\diff$, we are able to pass from the former $L_2$ norm to the latter norm that obeys orthogonality, and apply an analogue of the Rank Lemma~\ref{lem:rank-one-dim}.  Thus we derive a lower bound using any hard function $g$ whose corresponding matrix $G$ (built from orthogonal coefficients) is diagonal and high-rank.  And because the total number of polynomials is $L^*$, the rank argument now yields an exponential separation.

Based on this proof, we have much freedom in our choice of $g$.  By choosing its coefficients in the basis of multisymmetric powersum polynomials, it's easy to enforce the conditions that $G$ is diagonal and high-rank for variety of possible functions.  However, ensuring that $g$ is not pathological (i.e. that it is bounded and Lipschitz), and can be efficiently approximated in $\symlt$, requires a more careful choice.

\subsection{Sketch of Main Result (upper bound)}\label{sec:upper-bound-sketch}

It remains to approximate the hard function $g$ with a network from $\symlt$.  First we must make a choice of $g$ in particular.

Based on the lower bound proof, the desiderata for $g$ is that it is supported exclusively on terms of the form $\bp_{\alpha} \bp_{\alpha}$ over many values of $\alpha$, as this induces a diagonal and high-rank matrix $G$ in an analogue of Lemma~\ref{lem:rank-one-dim}.  Furthermore, by simple algebra one can confirm that $\bp_{\alpha}(X)\bp_{\alpha}(X) = \frac{1}{|\alpha|} \sum_{n,n'} \prod_{d=1}^D (x_{dn} x_{dn'})^{\alpha_d}$, so $g$ supported on these polynomials can clearly be written in the form of a network in $\symlt$.  This structure of $g$ guarantees difficult approximation, and is akin to the radial structure of the hard functions introduced in works on depth separation~\parencite{eldan2016power}.

We must however be careful in our choice of $g$: for the matrix $G$ to be high-rank, $g$ must be supported on exponentially many powersum polynomials.  But this could make $\|g\|_\infty$ exponentially large, and therefore challenging to approximate efficiently with a network from $\symlt$.

We handle this difficulty by defining $g$ in a different way.  We introduce a finite Blaschke product $\mu(\xi) = \frac{\xi - 1/4}{\xi/4 - 1}$, a function that analytically maps the unit complex circle to itself.  Then the choice
\begin{align}
        g(X) & = \sum_{n,n'=1}^N \prod_{d=1}^D \mu(x_{dn} x_{dn'}) 
\end{align}
ensures that $\|g\|_\infty$, $\|g\|_\A$, and $\mathrm{Lip}(g)$ are all polynomial in $N,D,\frac{1}{\epsilon}$ for $\epsilon$ approximation error (see Lemma~\ref{lem:g-prop}).  Furthermore, again from simple algebra it is clear that $g$ is only supported on terms of the form $\bp_{\alpha} \bp_{\alpha}$.  So it remains to show that the induced diagonal matrix $G$ is effectively high rank, which follows from expanding the Blaschke products.

Satisfied that this choice of $g$ will meet the desiderata for the lower bound, and has no pathological behavior, it remains to construct $f\in\symlt$ for $L=1$ that approximates $g$.  That is, choose $\psi_1$ and $\rho$ so that $g(X) \approx \rho\left(\sum_{n,n'=1}^N \psi_1(x_n, x_{n'})\right)$.  Clearly we may take $\rho$ to be the identity, and $\psi_1(x_n, x_{n'})$ to approximate $\prod_{d=1}^D \mu(x_{dn} x_{dn'})$, which is straightforwardly calculated in depth $O(\log D)$ by performing successive multiplications in a binary-tree like structure (see Theorem~\ref{thm:upper-bound-explicit}).

Ultimately, we use a slight variant of this function for the formal proof.  Because the orthogonality of our newly introduced inner product $\langle \cdot, \cdot \rangle_*$ only holds for low-degree polynomials, we must truncate high-degree terms of $g$; we confirm in Appendix \ref{sec:upper-bound} that this truncation nevertheless preserves the properties we care about.
\section{Discussion}

In this work, we've demonstrated how symmetric width captures more of the expressive power of symmetric networks than depth when restricted to analytic activations, by evincing an exponential separation between two of the most common architectures that enforce permutation invariance.

The most unusual property of this result is the complete independence of depth, owing to the unique orthogonality properties of the restricted Hall inner product when paired with the assumption of analyticity.  This stands in contrast to the case of vanilla neural networks, for which separations beyond small depth would resolve open questions in circuit complexity suspected to be quite hard~\parencite{vardi2021size}.  Furthermore, the greater dependence on width than depth is a unique property to symmetric networks, whereas the opposite is true for vanilla networks~\parencite{vardi2022width}.

A natural extension would be to consider the simple equivariant layers introduced in~\textcite{zaheer2017deep}, which we suspect will not substantially improve approximation power of $\syml$.  Furthermore, allowing for multiple such equivariant layers, this network becomes exactly akin to a Graph Convolutional Network ~\parencite{kipf2016semi} on a complete graph, whereas $\symlt$ corresponds to a message passing network~\parencite{gilmer2017neural} as it is capable of interpreting edge features.

\subsection{Limitations}\label{sec:limitations}

The major limitation of this result is the restriction to analytic functions.  Although analytic symmetric functions nevertheless appear crucially in the study of exactly solvable quantum systems~\parencite{langmann2005method, beau2021parent}, this assumption may be be overly strict for general problems of learning symmetric functions.  We nevertheless conjecture that these bounds will still hold even allowing for non-analytic activations, and consider this an exciting question for future work.  Additionally, whether the hard function $g$ can be efficiently learned with gradient descent remains unclear, and future work could touch on the learnability.
\paragraph{Acknowledgements:}
This work has been partially supported by the Alfred P. Sloan Foundation, NSF RI-1816753, NSF CAREER CIF-1845360, and NSF CCF-1814524.

\printbibliography

@article{segol2019universal,
  title={On universal equivariant set networks},
  author={Segol, Nimrod and Lipman, Yaron},
  journal={arXiv preprint arXiv:1910.02421},
  year={2019}
}

@inproceedings{lee2019set,
  title={Set transformer: A framework for attention-based permutation-invariant neural networks},
  author={Lee, Juho and Lee, Yoonho and Kim, Jungtaek and Kosiorek, Adam and Choi, Seungjin and Teh, Yee Whye},
  booktitle={International Conference on Machine Learning},
  pages={3744--3753},
  year={2019},
  organization={PMLR}
}

@article{pfau2020ab,
  title={Ab initio solution of the many-electron Schr{\"o}dinger equation with deep neural networks},
  author={Pfau, David and Spencer, James S and Matthews, Alexander GDG and Foulkes, W Matthew C},
  journal={Physical Review Research},
  volume={2},
  number={3},
  pages={033429},
  year={2020},
  publisher={APS}
}

@inproceedings{ma2018attend,
  title={Attend and interact: Higher-order object interactions for video understanding},
  author={Ma, Chih-Yao and Kadav, Asim and Melvin, Iain and Kira, Zsolt and AlRegib, Ghassan and Graf, Hans Peter},
  booktitle={Proceedings of the IEEE Conference on Computer Vision and Pattern Recognition},
  pages={6790--6800},
  year={2018}
}

@article{santoro2017simple,
  title={A simple neural network module for relational reasoning},
  author={Santoro, Adam and Raposo, David and Barrett, David G and Malinowski, Mateusz and Pascanu, Razvan and Battaglia, Peter and Lillicrap, Timothy},
  journal={Advances in neural information processing systems},
  volume={30},
  year={2017}
}

@article{wagstaff2022universal,
  title={Universal approximation of functions on sets},
  author={Wagstaff, Edward and Fuchs, Fabian B and Engelcke, Martin and Osborne, Michael A and Posner, Ingmar},
  journal={Journal of Machine Learning Research},
  volume={23},
  number={151},
  pages={1--56},
  year={2022}
}

@article{beau2021parent,
  title={Parent Hamiltonians of Jastrow wavefunctions},
  author={Beau, Mathieu and del Campo, Adolfo},
  journal={SciPost Physics Core},
  volume={4},
  number={4},
  pages={030},
  year={2021}
}

@inproceedings{wagstaff2019limitations,
  title={On the limitations of representing functions on sets},
  author={Wagstaff, Edward and Fuchs, Fabian and Engelcke, Martin and Posner, Ingmar and Osborne, Michael A},
  booktitle={International Conference on Machine Learning},
  pages={6487--6494},
  year={2019},
  organization={PMLR}
}

@article{murphy2018janossy,
  title={Janossy pooling: Learning deep permutation-invariant functions for variable-size inputs},
  author={Murphy, Ryan L and Srinivasan, Balasubramaniam and Rao, Vinayak and Ribeiro, Bruno},
  journal={arXiv preprint arXiv:1811.01900},
  year={2018}
}

@article{zaheer2017deep,
  title={Deep sets},
  author={Zaheer, Manzil and Kottur, Satwik and Ravanbakhsh, Siamak and Poczos, Barnabas and Salakhutdinov, Russ R and Smola, Alexander J},
  journal={Advances in neural information processing systems},
  volume={30},
  year={2017}
}

@inproceedings{qi2017pointnet,
  title={Pointnet: Deep learning on point sets for 3d classification and segmentation},
  author={Qi, Charles R and Su, Hao and Mo, Kaichun and Guibas, Leonidas J},
  booktitle={Proceedings of the IEEE conference on computer vision and pattern recognition},
  pages={652--660},
  year={2017}
}

@article{vardi2022width,
  title={Width is Less Important than Depth in ReLU Neural Networks},
  author={Vardi, Gal and Yehudai, Gilad and Shamir, Ohad},
  journal={arXiv preprint arXiv:2202.03841},
  year={2022}
}

@inproceedings{eldan2016power,
  title={The power of depth for feedforward neural networks},
  author={Eldan, Ronen and Shamir, Ohad},
  booktitle={Conference on learning theory},
  pages={907--940},
  year={2016},
  organization={PMLR}
}

@inproceedings{rydh2007minimal,
  title={A minimal set of generators for the ring of multisymmetric functions},
  author={Rydh, David},
  booktitle={Annales de l'institut Fourier},
  volume={57},
  number={6},
  pages={1741--1769},
  year={2007}
}

@book{macdonald1998symmetric,
  title={Symmetric functions and Hall polynomials},
  author={Macdonald, Ian Grant},
  year={1998},
  publisher={Oxford university press}
}

@article{maron2019provably,
  title={Provably powerful graph networks},
  author={Maron, Haggai and Ben-Hamu, Heli and Serviansky, Hadar and Lipman, Yaron},
  journal={Advances in neural information processing systems},
  volume={32},
  year={2019}
}

@article{domokos2007vector,
  title={Vector invariants of a class of pseudo-reflection groups and multisymmetric syzygies},
  author={Domokos, M{\'a}ty{\'a}s},
  journal={arXiv preprint arXiv:0706.2154},
  year={2007}
}

@inproceedings{safran2017depth,
  title={Depth-width tradeoffs in approximating natural functions with neural networks},
  author={Safran, Itay and Shamir, Ohad},
  booktitle={International conference on machine learning},
  pages={2979--2987},
  year={2017},
  organization={PMLR}
}

@article{venturi2021depth,
  title={Depth separation beyond radial functions},
  author={Venturi, Luca and Jelassi, Samy and Ozuch, Tristan and Bruna, Joan},
  journal={arXiv preprint arXiv:2102.01621},
  year={2021}
}

@inproceedings{daniely2017depth,
  title={Depth separation for neural networks},
  author={Daniely, Amit},
  booktitle={Conference on Learning Theory},
  pages={690--696},
  year={2017},
  organization={PMLR}
}

@article{vaswani2017attention,
  title={Attention is all you need},
  author={Vaswani, Ashish and Shazeer, Noam and Parmar, Niki and Uszkoreit, Jakob and Jones, Llion and Gomez, Aidan N and Kaiser, {\L}ukasz and Polosukhin, Illia},
  journal={Advances in neural information processing systems},
  volume={30},
  year={2017}
}

@article{langmann2005method,
  title={A method to derive explicit formulas for an elliptic generalization of the Jack polynomials},
  author={Langmann, Edwin},
  journal={arXiv preprint math-ph/0511015},
  year={2005}
}

@inproceedings{vardi2021size,
  title={Size and depth separation in approximating benign functions with neural networks},
  author={Vardi, Gal and Reichman, Daniel and Pitassi, Toniann and Shamir, Ohad},
  booktitle={Conference on Learning Theory},
  pages={4195--4223},
  year={2021},
  organization={PMLR}
}

@article{bassey2021survey,
  title={A survey of complex-valued neural networks},
  author={Bassey, Joshua and Qian, Lijun and Li, Xianfang},
  journal={arXiv preprint arXiv:2101.12249},
  year={2021}
}

@inproceedings{telgarsky2016benefits,
  title={Benefits of depth in neural networks},
  author={Telgarsky, Matus},
  booktitle={Conference on learning theory},
  pages={1517--1539},
  year={2016},
  organization={PMLR}
}

@article{kipf2016semi,
  title={Semi-supervised classification with graph convolutional networks},
  author={Kipf, Thomas N and Welling, Max},
  journal={arXiv preprint arXiv:1609.02907},
  year={2016}
}

@article{scarselli2008graph,
  title={The graph neural network model},
  author={Scarselli, Franco and Gori, Marco and Tsoi, Ah Chung and Hagenbuchner, Markus and Monfardini, Gabriele},
  journal={IEEE transactions on neural networks},
  volume={20},
  number={1},
  pages={61--80},
  year={2008},
  publisher={IEEE}
}

@article{xu2018powerful,
  title={How powerful are graph neural networks?},
  author={Xu, Keyulu and Hu, Weihua and Leskovec, Jure and Jegelka, Stefanie},
  journal={arXiv preprint arXiv:1810.00826},
  year={2018}
}

@article{diaconis1994eigenvalues,
  title={On the eigenvalues of random matrices},
  author={Diaconis, Persi and Shahshahani, Mehrdad},
  journal={Journal of Applied Probability},
  volume={31},
  number={A},
  pages={49--62},
  year={1994},
  publisher={Cambridge University Press}
}

@book{garnett2007bounded,
  title={Bounded analytic functions},
  author={Garnett, John},
  volume={236},
  year={2007},
  publisher={Springer Science \& Business Media}
}

@inproceedings{gilmer2017neural,
  title={Neural message passing for quantum chemistry},
  author={Gilmer, Justin and Schoenholz, Samuel S and Riley, Patrick F and Vinyals, Oriol and Dahl, George E},
  booktitle={International conference on machine learning},
  pages={1263--1272},
  year={2017},
  organization={PMLR}
}

@article{chen2020can,
  title={Can graph neural networks count substructures?},
  author={Chen, Zhengdao and Chen, Lei and Villar, Soledad and Bruna, Joan},
  journal={Advances in neural information processing systems},
  volume={33},
  pages={10383--10395},
  year={2020}
}


\newpage
\appendix
\appendixpage
\startcontents[sections]
\printcontents[sections]{l}{1}{\setcounter{tocdepth}{2}}
\section{Preliminaries}\label{sec:prelim}

\subsection{Notation}

We'll use $\N$ to denote the naturals including $0$.
The indicator function for the condition $x = y$ is written as $\ind_{x = y}$.
Given an integer $\emph{weak composition}$ $\alpha \in \N^D$, we will often consider the multidimensional polynomial $z^\alpha = \prod_{d=1}^D z_{d}^{\alpha_d}$.
For two vectors $x, x' \in \mathbb{C}^D$, we denote their elementwise product by $x \circ x'$.

\subsection{Inner Products}

We introduce two $L_2$ inner products (defined with respect to probability measures) we'll use throughout the work.
For symmetric functions $f, g: \mathbb{C}^{N} \rightarrow \mathbb{C}$, define:
\begin{align}
    \langle f, g \rangle_\V = \frac{1}{(2\pi)^NN!} \int_{[0,2\pi]^N}f(e^{i\btheta}) \conj{g(e^{i\btheta})} |V(e^{i\btheta})|^2 d\btheta~,
\end{align}
where for $z \in \mathbb{C}^N$, we have the Vandermonde determinant
\begin{align}
    V(z) = \prod_{1 \leq i < j \leq N} (z_j - z_i)~.
\end{align}

This inner product is well-known in the theory of symmetric polynomials, as a finite-variable analogue of the Hall inner product~\parencite{macdonald1998symmetric}.  Equivalently, if we let $V$ denote the joint density of eigenvalues of a Haar-distributed unitary matrix in $\mathbb{C}^{N \times N}$, it is known~\parencite{diaconis1994eigenvalues} that this inner product may be written as
\begin{align}
    \langle f, g \rangle_V = \mathbb{E}_{y \sim V}\left[f(y) \overline{g(y)} \right]~.
\end{align}
For arbitrary functions $f, g: \mathbb{C}^D \rightarrow \mathbb{C}$, we also consider the $L_2$ inner product given as an expectation over $D$ random variables
\begin{align}
    \langle f, g \rangle_\S &= \frac{1}{(2\pi)^D} \int_{[0,2\pi]^D} f(e^{i\btheta}) \conj{g(e^{i\btheta})} d\btheta \\
    & = \mathbb{E}_{q \sim (\S)^D}\left[f(q)\overline{g(q)} \right]~,
\end{align}
with the notation $q \sim (\S)^D$ meaning each entry of $q$ is i.i.d. uniform on $\S$.

For this inner product, we will introduce the following notation.  For a multi-index $\alpha \in \N^D$ and a dummy variable $q$ of dimension $D$, we let $q^\alpha$ denote the polynomial function $z \mapsto z^\alpha$.  Then it's clear that
\begin{align}
    \langle q^\alpha, q^\beta \rangle_\S = \ind_{\alpha = \beta}~.
\end{align}
Note that we will consider this inner product over varying dimensions throughout the paper, but it will be clear from context the dimension, i.e. how many i.i.d. random variables uniform on $S^1$ we are sampling over.

\subsection{Symmetric Polynomials}

We remind the notation from the main body: $p_0(x) = 1$, and for $k \in \N \setminus\{0\}$ and any partition $\lambda$:
\begin{align}
    p_k(x) & = \frac{1}{\sqrt{k}} \sum_{n=1}^N x_n^k \\
    p_\lambda(x) &= \prod_i p_{\lambda_i}(x)~.
\end{align}

We will also sometimes use set notation to index products of powersums.  For example, $p_{\{2,1\}} = p_2 p_1 = p_1 p_2$.

Finally, we need the notation that if $n_t$ denotes the number of times $t$ appears in $\lambda$, then $z_\lambda = \prod_t n_t!$.  Note that this definition of $z_\lambda$ is slightly different that most texts, as we're considering the normalized powersums.

Then we can state Theorem~\ref{thm:hall-inner-product} explicitly:
\begin{theorem}\label{thm:hall-inner-product-explicit}[{\cite[Chapter VI (9.10)]{macdonald1998symmetric}} ]
    For partitions $\lambda, \mu$ with $|\lambda| \leq N$:
    \begin{align}
        \langle p_\lambda, p_\mu \rangle_\V = z_\lambda \ind_{\lambda = \mu}~.
    \end{align}
\end{theorem}

\subsection{Multisymmetric Polynomials}

When $D > 1$, in order to approximate our network with polynomials, we introduce the multivariate analog of symmetric polynomials.  For example, suppose $D = 2$, and we write our set elements the following way:
$$ X = \left\{\begin{bmatrix}
            y_1 \\
            z_1
         \end{bmatrix}, \begin{bmatrix}
            y_2 \\
            z_2
         \end{bmatrix}, \dots \begin{bmatrix}
            y_N \\
            z_N
         \end{bmatrix}\right\} $$
Then a basis of symmetric functions is given by the multisymmetric power sum  polynomials, some examples:
\begin{align}
    \bp_{(2,3)}(X) & = \frac{1}{\sqrt{2 + 3}}\sum_n y_n^2 z_n^3 \\
    \bp_{(4,1)}(X) & = \frac{1}{\sqrt{4 + 1}}\sum_n y_n^4 z_n^1 ~.
\end{align}
For general $N$ and $D$, our input is $X \in \mathbb{C}^{D \times N}$ where we want functions that are invariant to permuting the columns $x_n$ of this matrix.  Note that we write scalar entries of this matrix as $x_{dn}$.

\begin{definition}
For a multi-index $\alpha \in \N^{D}$, the 
\emph{normalized multisymmetric powersum polynomial} is defined as:
\begin{align}
    \bp_{\alpha}(X) & = \frac{1}{\sqrt{|\alpha|}} \sum_n x_n^\alpha \\
    &= \frac{1}{\sqrt{|\alpha|}} \sum_n \prod_d x_{dn}^{\alpha_d}
\end{align}
with $\bp_{0} = 1$.

\end{definition}

An algebraic basis of symmetric functions in this setting is given by all $\bp_{\alpha}$ for all $|\alpha| \leq N$, where $|\alpha| = \sum_{d} \alpha_d$ (for a proof see~\textcite{rydh2007minimal}).

We remind the notation from the introduction, where $L^*(N, D) = |\{\alpha \in \mathbb{N}^D: |\alpha| \leq N\}| = \binom{N + D}{N} - 1$ is the size of this algebraic basis (discouting the constant polynomial).
Intuitively then it's clear why $L \geq L^*$ will make $\syml$ a universal approximator, as each of the $L$ symmetric features $\{\phi_l\}_{l=1}^L$ will calculate one of these basis elements.
\section{One Dimensional Set Elements}

We will first consider the setting where $D = 1$, i.e. each set element is a scalar.  In this setting, we will amend notation slightly so that we consider symmetric functions $f$ acting on $x \in \mathbb{C}^N$, where each $x_n$ is a scalar set element.

\subsection{Projection Lemma}

Let us remind $\P_1$ to be the orthogonal projection onto $span(\{p_t : 1 \leq t \leq N/2\})$, and $\P_2$ to be the orthogonal projection onto $span(\{p_tp_{t'} : 1 \leq t,t' \leq N/2\})$.

\begin{lemma}\label{lem:proj}
    Given any $f \in \syml$, we may choose coefficients $v_{ij}$ over $i \leq j \leq L$, and symmetric polynomials $\phi_i$ over $i \leq L$, such that:
    \begin{align}
        \P_2 f = \sum_{i\leq j}^L v_{ij} (\P_1 \phi_i) (\P_1\phi_j)~.
    \end{align}
\end{lemma}
\begin{proof}

Consider the general parameterization of $f$ given in Equation~\ref{eq:symnn}.  Because all network activations are analytic, we can write all maps parameterizing $f$ by power series.

Note that the inner product $\langle \cdot, \cdot \rangle_\V$ integrates over a compact domain, therefore the projection $\P_2 f$ will be determined by the value of $f$ restricted to that domain.  Thus, all power series in the sequel will converge uniformly and we may freely interchange infinite sums with each other as well as with inner products.

Explicitly, to parameterize $f$ we write $\psi_l(x_n) = c_{l0} + \sum_{k=1}^\infty \frac{c_{lk}}{\sqrt{k}} x_n^k$ so that $\phi_l(x) = \sum_{n=1}^N \psi_l(x_n) = Nc_{l0} + \sum_{k=1}^\infty c_{lk} \tp_k(x)$.

Because $\rho$ is also given as a power series, it can be equivalently written as a power series with all variables having constant offsets.  So we can subtract the constant terms from every $\phi_l$ and write:
\begin{align}
    \rho(y) &= \sum_{\comp \in \mathbb{N}^L} v_\comp y^\comp ~,\\
    \phi_l & = \sum_{k=1}^\infty c_{lk} \tp_k~,
\end{align}
where $y^\comp = \prod_{n=1}^N y_n^{\comp_n}$.  Hence
\begin{align}
    f & = \rho(\phi_1, \dots, \phi_L) = \sum_{\comp} v_\comp \phi^\comp ~.
\end{align}
We proceed to calculate $\P_2 f$.  To begin, consider $\langle p_t p_{t'}, \phi^\comp \rangle$ for any choice of indices $1 \leq t, t' \leq N/2$.
To illustrate, suppose $\comp_i = \comp_j = \comp_k = 1$ and $\comp$ is $0$ everywhere else.  Then we may write
\begin{align}
    \langle p_t p_{t'}, \phi^\comp \rangle_\V = \langle p_t p_{t'}, \phi_i \phi_j \phi_k \rangle_\V & = \sum_{i'=1}^\infty \sum_{j'=1}^\infty \sum_{k'=1}^\infty c_{ii'} c_{jj'} c_{kk'} \langle p_t p_{t'}, p_{i'} p_{j'} p_{k'} \rangle_\V = 0~.
\end{align}
In other words, after distributing the product $\phi_i \phi_j \phi_k$, we are left with a sum of terms of the form $p_{i'} p_{j'} p_{k'}$.  So treated as partitions, we clearly have $\{i', j', k'\} \neq \{t, t'\}$, where all these indices are positive.  Thus, because $t + t' \leq N$, we can apply the orthogonality property of the inner product to conclude $\langle p_t p_{t'}, p_{i'} p_{j'} p_{k'} \rangle_\V = 0$.

By similar logic, $\langle p_t p_{t'}, \phi^\comp \rangle = 0$ whenever $|\comp| \neq 2$, so we may cancel all such terms in the expansion of $f$ to get
\begin{align*}
    \P_2 f = \P_2 \left(\sum_{\comp \in \N^L} v_\comp \phi^\comp\right)  = \sum_{|\comp| = 2} v_\comp \P_2 \phi^\comp~.
\end{align*}

Here we can simplify notation.  Let $\{e_i\}_{i=1}^L$ denote the standard basis vectors in dimension $L$. Every $\comp \in \N^L$ with $|\comp| = 2$ can be written as $\comp = e_i + e_j$, so let $v_{ij} := v_{e_i + e_j}$.  Then we can rewrite:
$$
    \P_2 f = \sum_{i \leq j}^L v_{ij} \P_2 \phi_i \phi_j~.
$$
Finally, note again by orthogonality we have that $\P_2 (p_{i'} p_{j'}) = 0$ if it is not the case that $1 \leq i', j' \leq N/2$.  So observe that we may pass from $\P_2$ to $\P_1$:
\begin{align}
    \P_2 \phi_i \phi_j &= \P_2 \left(\sum_{i'=1}^\infty c_{ii'} p_{i'} \right) \left(\sum_{j'=1}^\infty c_{jj'} p_{j'} \right) \\
    & = \P_2 \sum_{i'=1}^\infty \sum_{j'=1}^\infty c_{ii'} c_{jj'} p_{i'} p_{j'} \\
    & = \sum_{i'=1}^{N/2} \sum_{j'=1}^{N/2} c_{ii'} c_{jj'} p_{i'} p_{j'} \\
    & = \left(\sum_{i'=1}^{N/2} c_{ii'} p_{i'} \right) \left(\sum_{j'=1}^{N/2} c_{jj'} p_{j'} \right) \\
    & = (\P_1 \phi_i) (\P_1 \phi_j)~.
\end{align}
So ultimately we get
\begin{align}
    \P_2 f & = \sum_{i\leq j}^L v_{ij} (\P_1 \phi_i) (\P_1 \phi_j)~.
\end{align}
\end{proof}

\subsection{Rank Lemma}

The following lemma is a generalization of the the Rank Lemma~\ref{lem:rank-one-dim}, which we will use for both the one- and high-dimensional cases.  Ultimately, for an inner product $\langle \cdot, \cdot \rangle$ with certain orthogonality properties, it allows us to pass from function error $\|f-g\|^2$ to Frobenius norm error $\|F - G\|_F^2$ for some induced matrices $F, G$.

\begin{lemma}\label{lem:rank}
    Consider a commutative algebra equipped with an inner product, and a set of elements $\{p_t\}_{t=1}^T$.  Suppose the terms $p_{\{t,t'\}} = p_t p_{t'}$, indexed by sets $\{t,t'\}$, are pairwise orthogonal, and normalized such that
    \begin{align*}
        \|p_{t} p_{t'}\|^2 \geq \begin{cases} 
           1 & t \neq t' \\
           2 & t = t'
       \end{cases}
    \end{align*}
    
    Consider the terms:
    \begin{align*}
        \phi_l & = \sum_{t=1}^T c_{lt} p_t ~,\\
        f & = \sum_{l \leq l'}^L \frac{v_{ll'}}{1 + \ind_{l=l'}} \phi_l \phi_{l'} ~,\\
        g & = \sum_{t \leq t'}^T \frac{g_{tt'}}{1 + \ind_{t=t'}} p_t p_{t'}~.
    \end{align*}
    
    Then we have the bound
    \begin{align}
        \|f - g\|^2 \geq \frac{1}{2} \|C^T V C - G\|_F^2~,
    \end{align}
    where $C_{lt} = c_{lt}, V_{ll'} = v_{ll'}, G_{tt'} = g_{tt'}$, where we define $V$ and $G$ to be symmetric.
\end{lemma}

\begin{proof}

    To begin, we calculate inner products for $t \neq t'$:
    \begin{align}
       \left\langle f, \frac{p_{\{t,t'\}}}{\|p_{\{t,t'\}}\|} \right\rangle & = \frac{1}{\|p_{\{t,t'\}}\|}\left\langle \sum_{l \leq l'}^L \sum_{t,t'=1}^T \frac{v_{ll'}}{1 + \ind_{l=l'}}  c_{lt} c_{l't'} p_t p_{t'}, p_t p_{t'} \right\rangle \\
        & = \|p_{t} p_{t'}\| \sum_{l \leq l'}^L \frac{v_{ll'}}{1 + \ind_{l=l'}}  (c_{lt} c_{l't'} + c_{lt'} c_{l't})\\
        & = \|p_{t} p_{t'}\| \left( \sum_{l = l'}^L \frac{v_{ll}}{2}  (c_{lt} c_{lt'} + c_{lt'} c_{lt}) + \sum_{l < l'}^L v_{ll'}  (c_{lt} c_{l't'} + c_{lt'} c_{l't}) \right)\\
        & = \|p_{t} p_{t'}\| \left( \sum_{l = l'}^L v_{ll} c_{lt}c_{lt'} + \sum_{l < l'}^L v_{ll'}  (c_{lt} c_{l't'} + c_{lt'} c_{l't}) \right)~.
    \end{align}
    Defining $v_{ll'} = v_{l'l}$, we may reindex and write the second sum as:
    \begin{align}
        \sum_{l < l'}^L v_{ll'}  (c_{lt} c_{l't'} + c_{lt'} c_{l't}) & = \sum_{l < l'}^L v_{ll'}c_{lt} c_{l't'} + \sum_{l < l'}^L v_{ll'} c_{lt'} c_{l't} \\
        & = \sum_{l < l'}^L v_{ll'}c_{lt} c_{l't'} + \sum_{l > l'}^L v_{ll'} c_{lt} c_{l't'}~.
    \end{align}
    
    So putting this together we get
    \begin{align*}
        \left\langle f, \frac{p_{\{t,t'\}}}{\|p_{\{t,t'\}}\|} \right\rangle = \|p_{t} p_{t'}\| \left( \sum_{l, l'}^L v_{ll'} c_{lt}c_{l't'} \right) =  \|p_{t} p_{t'}\| [C^TVC]_{t,t'}~.
    \end{align*}
    By a similar calculation we conclude:
    \begin{align*}
        \left\langle f, \frac{p_{\{t,t\}}}{\|p_{\{t,t\}}\|} \right\rangle = \frac{\|p_{t} p_{t}\|}{2} [C^TVC]_{t,t}~.
    \end{align*}
    For $g$, we can directly calculate:
    \begin{align}
        \left\langle g, \frac{p_{\{t,t'\}}}{\|p_{\{t,t'\}}\|} \right\rangle & = \|p_{t} p_{t'}\|[G]_{t,t'} \\
        \left\langle g, \frac{p_{\{t,t\}}}{\|p_{\{t,t\}}\|} \right\rangle & = \frac{\|p_{t} p_{t}\|}{2} [G]_{t,t}~.
    \end{align}
    Finally, by Parseval's Theorem we calculate:
    \begin{align}
        \|f - g\|^2 & = \sum_{t} \left( \left\langle f, \frac{p_{\{t,t\}}}{\|p_{\{t,t\}}\|} \right\rangle - \left\langle g, \frac{p_{\{t,t\}}}{\|p_{\{t,t\}}\|} \right\rangle \right)^2 + \sum_{t < t'}^T \left( \left\langle f, \frac{p_{\{t,t'\}}}{\|p_{\{t,t'\}}\|} \right\rangle - \left\langle g, \frac{p_{\{t,t'\}}}{\|p_{\{t,t'\}}\|} \right\rangle \right)^2 \\
        & = \sum_{t} \left( \left\langle f, \frac{p_{\{t,t\}}}{\|p_{\{t,t\}}\|} \right\rangle - \left\langle g, \frac{p_{\{t,t\}}}{\|p_{\{t,t\}}\|} \right\rangle \right)^2 + \frac{1}{2} \sum_{t \neq t'}^T \left( \left\langle f, \frac{p_{\{t,t'\}}}{\|p_{\{t,t'\}}\|} \right\rangle - \left\langle g, \frac{p_{\{t,t'\}}}{\|p_{\{t,t'\}}\|} \right\rangle \right)^2 \\
        & = \sum_{t}^T \frac{\|p_{\{t,t\}}\|^2}{4}  [C^TVC - G]_{t,t}^2 + \frac{1}{2} \sum_{t \neq t'}^T \|p_{\{t,t'\}}\|^2 \cdot [C^TVC - G]_{t,t'}^2 \\
        & \geq \frac{1}{2} \sum_{t}^T [C^TVC - G]_{t,t}^2 + \frac{1}{2} \sum_{t \neq t'}^T [C^TVC - G]_{t,t'}^2~,
    \end{align}
    where in the last line we use our assumption on the lower bound of $\|p_{\{t,t'\}}\|^2$ and $\|p_{\{t,t\}}\|^2$.
    Hence:
    \begin{align}
        \|f - g\|^2 \geq \frac{1}{2} \|C^TVC - G\|_F^2~.
    \end{align}
\end{proof}

\subsection{Proof of one-dimensional Lower Bound}

\begin{theorem}
    Let $D = 1$.  Then using the Vandermonde $L_2$ inner product over symmetric polynomials
    \begin{align}
        \max_{\|g\|_\V = 1} \min_{f\in\syml} \|f - g\|_\V^2 \geq 1 - \frac{2L}{N}~.
    \end{align}
    In particular, for $L = \frac{N}{4}$ we recover a constant lower bound.
\end{theorem}

\begin{proof}

We first build our counterexample $g$ by choosing its coefficients in the powersum basis, say:
\begin{align}
    g & = \frac{1}{\sqrt{N}} \sum_{t=1}^{N/2}  \tp_t \tp_t~.
\end{align}
From orthogonality and the fact that $\|p_tp_t\|_\V^2 = 2$ it's clear that $\|g\|_\V = 1$, and note that $\P_2 g = g$.
Applying Lemma~\ref{lem:proj}, for any $f \in \syml$ we can write $\P_2 f$ in the form 
\begin{align}
    \P_2 f & = \sum_{i\leq j}^L v_{ij} (\P_1 \phi_i) (\P_1 \phi_j)~.
\end{align}
One may also confirm that the Vandermonde inner product satisfies the requirements of Lemma~\ref{lem:rank} when restricted to the range of $\P_2$, owing to the orthogonality property and the fact that for $1 \leq t, t' \leq N/2$:
    \begin{align*}
        \langle p_{t} p_{t'}, p_{t} p_{t'} \rangle_\V = \begin{cases} 
           1 & t \neq t' \\
           2 & t = t'
       \end{cases}
    \end{align*}
So we've met all the necessary requirements to apply Lemma~\ref{lem:rank} to $\P_2 f$ and $\P_2 g$, thus we have:
\begin{align}
    \min_{f\in\syml} \|f - g\|_\V^2 & \geq \min_{f\in\syml} \|\P_2 f - \P_2 g \|_\V^2 \\
    & \geq \min_{C,V} \frac{1}{2} \|C^TVC - 2 * \frac{1}{\sqrt{N}} I\|_F^2 \\
    & = \min_{C,V} \frac{1}{N/2} \|C^TVC - I\|_F^2~,
\end{align}
where the factor of $2$ appears based on the definition of the matrix $G$ in Lemma~\ref{lem:rank} 

Note that $CVC^T \in \mathbb{C}^{N/2 \times N/2}$, but $V \in \mathbb{C}^{L \times L}$.  So if $N/2 > L$, then $CVC^T$ is a rank-deficient approximation of the identity, and clearly we have
\begin{align}
    \min_{f\in\syml} \|f - g\|_\V^2 
    & \geq \frac{N/2 - L}{N/2} = 1 - \frac{2L}{N}~.
\end{align}
\end{proof}
\section{Exact statement of Main Result}

\subsection{Theorem Statement}

We begin by restating the main result, where for convenience we will change from $N$ set elements to $2N$.  

We introduce the notation $\hat{D} := \min\left(D, \lfloor\sqrt{N/2}\rfloor\right)$.
We also introduce the $L_2$ inner product
\begin{align}\label{eq:inner_prod_a}
    \langle f, g \rangle_\A = \mathbb{E}_{y \sim V; q,r \sim (S^1)^D} \left[f(X(y,q,r)) \conj{g(X(y,q,r))}\right]~,
\end{align}
where the set input $X(y,q,r) \in \mathbb{C}^{D \times 2N}$ with matrix entries $x_{dn}(y,q,r)$ is defined by:
\begin{align}
    x_{dn}(y,q,r) = \begin{cases} 
          q_d y_{n} & 1 \leq n \leq N ~,\\
          r_d y_{n-N} & N+1 \leq n \leq 2N~.
       \end{cases}
\end{align}

And we restate the activation assumption in this new notation:

\begin{assumption}\label{ass:act-real}
    The activation $\sigma : \mathbb{C} \rightarrow \mathbb{C}$ is analytic, and for a fixed $D, N$ there exist two-layer neural networks $f_1, f_2$ using $\sigma$, both with $O\left(D^2 + D \log \frac{D}{\epsilon}\right)$ width and $O(D \log D)$ bounded weights, such that:
        \begin{align}
        \sup_{|\xi| \leq 3} |f_1(\xi) - \xi^2| \leq \epsilon, \qquad
        \sup_{|\xi| \leq 3} \left|f_2(\xi) - \left(1 - (\xi/4)^{\min(D, \sqrt{N/2})}\right) \frac{\xi - 1/4}{\xi/4 - 1} \right| \leq \epsilon
    \end{align}
\end{assumption}

Then our main theorem is thusly:

\begin{theorem}[Exponential width-separation]\label{thm:main-result}
    Fix $2N$ and $D$ such that $\hat{D} > 1$, and consider set elements $X \in \mathbb{C}^{D \times 2N}$.  Define
    \begin{align}
        g(X) = -\frac{4N^2}{4^{\hat{D}}} + \sum_{n,n'=1}^{2N}\prod_{d=1}^{\hat{D}} \left(1 - (x_{dn}x_{dn'}/4)^{\hat{D}}\right)\frac{x_{dn}x_{dn'} - 1/4}{x_{dn}x_{dn'}/4 - 1} \\
    \end{align}
    and $g' = \frac{g}{\|g\|_\A}$.
    Then the following is true:
    \begin{itemize}
        \item For $L \leq N^{-2}\exp(O(\hat{D}))$,
        \begin{align}
            \min_{f \in \syml} \|f - g'\|_\A^2 \geq \frac{1}{12}~.
        \end{align}
        \item For $L = 1$, there exists $f \in \symlt$, parameterized with an activation $\sigma$ that satisfies Assumption~\ref{ass:act-real}, with width $poly(N,D,1/\epsilon)$, depth $O(\log D)$, and maximum weight magnitude $O(D \log D)$ such that over the unit torus:
        \begin{align}
            \|f - g'\|_\infty \leq \epsilon~.
        \end{align}
    \end{itemize}
\end{theorem}
\begin{remark}
Let us remark about one aspect that will ease exposition. In the sequel, we will assume $D \leq \sqrt{N/2}$ so that $\hat{D} = D$.  This is not a necessary assumption; in the case that $D > \sqrt{N/2}$, we can simply replace all instances of $D$ with $\hat{D}$ in the definition of $g$ and the subsequent proof.
Because the data distribution has each row of $X \in \mathbb{C}^{D \times 2N}$ is i.i.d., the proof goes through exactly.  Indeed, it would be equivalent to truncating each set vector to the first $\hat{D}$ elements.  This will only impact the bounds by replacing $D$ with $\hat{D}$, in which circumstances we will clearly state.
\end{remark}
\subsection{Proof Roadmap}

Let us roadmap the general proof.  

In Section~\ref{sec:l2-inner-product}, we justify the inner product $\langle \cdot, \cdot \rangle_\A$ and show it can be used to prove a high-dimensional analogue of the Projection Lemma (see Lemma~\ref{lem:projhigh}).  In Section~\ref{sec:diagonal-inner-product} we further introduce a second inner product, whose orthogonality properties (see Theorem~\ref{thm:bilinear-form}) allow us to apply the Rank Lemma~\ref{lem:rank}.  In Section~\ref{sec:actual-proof-of-lower-bound}, we combine these results to first prove a lower bound for a simple choice of hard function (see Theorem~\ref{thm:high-separation-easy}).   Because this simple choice is not suitable for demonstrating the upper bound, we then conclude by showing the hard function $g'$ also evinces a lower bound via a similar argument (see Theorem~\ref{thm:high-separation-hard}).

In Section~\ref{sec:construct-g}, we demonstrate the properties of the hard function $g$, by constructing the pieces of $g$ one by one and controlling their behavior, leading to Lemma~\ref{lem:g-prop} which yields all the properties we need about $g$ for the rest of the proof.

In Section~\ref{sec:upper-bound} we complete the proof of the upper bound.  Specifically, in Section~\ref{sec:exact-rep} we show how to write $g'$ exactly in an analogous form to $\symlt$, but using very specific activations.  In Section~\ref{sec:approx-rep} we write an approximation of this network in $\symlt$ using a given activation, and in Section~\ref{sec:proof-upper-bound} we control the error between these two networks.
\section{Lower Bound of Main Result}\label{sec:lower-bound-app}

\subsection{An $L_2$ inner product}\label{sec:l2-inner-product}

As discussed in Section~\ref{sec:lower-bound}, we must first define an appropriate $L_2$ inner product, before we can prove a lower bound on function approximation.  To that end, we will define an input distribution for the set inputs $X$.

Let us introduce several random variables: let $y \sim V$ as in the definition of the inner product $\langle \cdot, \cdot \rangle_\V$ over $N$ variables.  Let $q$ and $r$ be two random vectors of dimension $D$, with each entry $i.i.d.$ uniform on $S^1$.

Then we can define an input distribution for $X \in \mathbb{C}^{D \times 2N}$ with matrix entries $x_{dn}$:
\begin{align}
    x_{dn} = \begin{cases} 
          q_d y_{n} & 1 \leq n \leq N \\
          r_d y_{n-N} & N+1 \leq n \leq 2N~.
       \end{cases}
\end{align}

The point of this assignment is how it transforms multisymmetric power sums:
\begin{align}
    \bp_\alpha(X) & = \frac{1}{\sqrt{|\alpha|}}\sum_{n=1}^{2N} \prod_d x_{dn}^{\alpha_d} \\
    & = \frac{1}{\sqrt{|\alpha|}} \sum_{n=1}^{N} \prod_d y_{n}^{\alpha_d} q_d^{\alpha_d} + \frac{1}{\sqrt{|\alpha|}} \sum_{n=1}^{N} \prod_d y_{n}^{\alpha_d} r_d^{\alpha_d} \\
    & = p_{|\alpha|}(y) \cdot (q^\alpha + r^\alpha)~.
\end{align}
Then as stated before we have the inner product:
\begin{align}\label{eq:inner_prod_a}
    \langle f, g \rangle_\A = \mathbb{E}_{y \sim V, q \sim (S^1)^D, r \sim (\S)^D} \left[f(X) \conj{g(X)}\right]~.
\end{align}
From our choices above we may use separability to write $\langle \cdot, \cdot \rangle_\A$ in terms of previously introduced inner products.  For example:
\begin{align}
    \langle \bp_\alpha, \bp_\beta \rangle_\A & = \mathbb{E}_{y,q,r}\left[p_{|\alpha|}(y) (q^\alpha + r^\alpha) \conj{p_{\beta|}(y) (q^\beta + r^\beta)} \right] \\
    & = \mathbb{E}_y\left[p_{|\alpha|}(y)  \conj{p_{\beta|}(y)} \right] 
    \mathbb{E}_{q,r}\left[(q^\alpha + r^\alpha) \conj{(q^\beta + r^\beta)} \right] \\
    & = \langle p_{|\alpha|}, p_{|\beta|} \rangle_\V \cdot \langle q^\alpha + r^\alpha, q^\beta + r^\beta \rangle_\S~.
\end{align}
We can now observe this inner product grants a ``partial" orthogonality:
\begin{lemma}\label{lem:partialortho}
    Consider $\alpha, \beta \in \N^D$ with $1 \leq |\alpha|, |\beta| \leq N/2$.  Then for $\gamma_k \in \N^{D} \setminus \{0\}$, if $K \neq 2$
    \begin{align}
        \left\langle \tbp_\alpha \tbp_\beta, \prod_{k=1}^K \tbp_{\gamma_k} \right\rangle_{\A} = 0~.
    \end{align}
    Otherwise, for $K = 2$, we have:
    \begin{align}\label{eq:twoprod}
        \langle \tbp_\alpha \tbp_\beta, \tbp_\gamma \tbp_\delta \rangle_{\A} =
          2 \cdot (1 + \ind_{|\alpha| = |\beta|}) \cdot \ind_{\{|\alpha|, |\beta|\} = \{|\gamma|, |\delta|\}} \cdot (\ind_{\alpha + \beta = \gamma + \delta} + \ind_{(\alpha, \beta) = (\gamma, \delta)} + \ind_{(\alpha, \beta) = (\delta, \gamma)})~.
    \end{align}
\end{lemma}

\begin{proof}

By separability, we can confirm that 
\begin{align}
    \langle \tbp_\alpha \tbp_\beta, \prod_{k=1}^K \tbp_{\gamma_k} \rangle_{\A} = \langle p_{|\alpha|} p_{|\beta|}, \prod_{k=1}^K p_{|\gamma_k|} \rangle_\V \cdot C~,
\end{align}
    where $C$ is the value of the expectation on the random variables $q$ and $r$.  Thus if $K \neq 2$, because $|\alpha| + |\beta| \leq N$, this term is 0 by orthogonality of the Vandermonde inner product.
    
For the $K=2$ case, we begin again by using separability:
\begin{align}
    \langle \tbp_\alpha \tbp_\beta, \tbp_\gamma \tbp_\delta \rangle_{\A} & = \left\langle \tp_{|\alpha|} \tp_{|\beta|} , \tp_{|\gamma|} \tp_{|\delta|} \right\rangle_\V \cdot \left\langle (q^\alpha + r^\alpha) (q^\beta + r^\beta), (q^\gamma + r^\gamma) (q^\delta + r^\delta)  \right\rangle_\S~. 
\end{align}
Let's consider first the inner product of power sums.  Plugging in the definition of the normalizing constant $z_\lambda$ gives:
\begin{align*}
    \left\langle \tp_{|\alpha|} \tp_{|\beta|} , \tp_{|\gamma|} \tp_{|\delta|} \right\rangle_\V = (1 + \ind_{|\alpha| = |\beta|}) \cdot \ind_{\{|\alpha|, |\beta|\} = \{|\gamma|, |\delta|\}}~.
\end{align*}

Consider now the second inner product term.  Noting that each element $q_d, r_d$ is i.i.d. uniform on the unit circle, orthogonality of the Fourier basis implies we can calculate this inner product  by only including terms with matching exponents. Bearing in mind that $\alpha, \beta, \gamma, \delta \neq 0$, we must always have terms of the form $\langle q^{\alpha + \beta}, q^\gamma r^\delta \rangle_\S = 0$, and therefore we distribute and calculate:
\begin{align*}
    & \left\langle q^{\alpha + \beta} + q^\alpha r ^\beta + q^{\beta} r^\alpha + r^{\alpha + \beta}, q^{\gamma + \delta} + q^\gamma r ^\delta + q^{\delta} r^\gamma + r^{\gamma + \delta} \right\rangle_\S \\
    & = \langle q^{\alpha + \beta}, q^{\gamma + \delta} \rangle_\S + \langle q^\alpha r ^\beta + q^{\beta} r^\alpha, q^\gamma r ^\delta + q^{\delta} r^\gamma \rangle_\S + \langle r^{\alpha + \beta}, r^{\gamma + \delta} \rangle_\S  \\
    & = 2 \cdot \ind_{\alpha + \beta = \gamma + \delta} + 2 \cdot \ind_{(\alpha, \beta) = (\gamma, \delta)} + 2 \cdot \ind_{(\alpha, \beta) = (\delta, \gamma)}~.
\end{align*}

Collecting the terms of both products and evaluating the indicator functions under all cases gives the result.
\end{proof}

Looking at Equation~\ref{eq:twoprod}, we can see inner product $\langle \cdot, \cdot \rangle_\A$ does not grant full orthogonality.  The inner product gives orthogonality between powersum products of different lengths, but $\langle \tbp_\alpha \tbp_\beta, \tbp_\gamma \tbp_\delta \rangle_{\A}$ can be non-zero if $\alpha + \beta = \gamma + \delta$, even in the cases where $\{\alpha, \beta\} \neq \{\gamma, \delta\}$.

Nevertheless, this inner product still suffices to prove a similar result about projection for the $D > 1$ case.

Let $\P_1$ be the orthogonal projection onto $span(\{\tbp_\alpha : 1 \leq |\alpha|, |\beta| \leq N/2 \})$ and $\P_2$ be the orthogonal projection onto $span(\{\tbp_\alpha \tbp_\beta : 1 \leq |\alpha|, |\beta| \leq N/2 \})$.  Here by orthogonal, we mean with respect to $\langle \cdot, \cdot \rangle_\A$.  

\begin{lemma}\label{lem:projhigh}
    Given any $f \in \syml$ with $D > 1$, we may choose coefficients $v_{ij}$ over $i \leq j \leq L$, and multisymmetric polynomials $\phi_i$ over $i \leq L$, such that:
    \begin{align}
        \P_2 f = \sum_{i\leq j}^L v_{ij} (\P_1 \phi_i) (\P_1\phi_j)~.
    \end{align}
\end{lemma}

\begin{proof}
As in Lemma~\ref{lem:proj}, if we approximate $\psi_l(x_n) = c_{l0} + \sum_{\alpha \neq 0} \frac{c_{l\alpha}}{\sqrt{|\alpha|}} x_n^\alpha$, then symmetrizing gives $\phi_l(X) = Nc_{l0} + \sum_{\alpha \neq 0} c_{l\alpha} \bp_{\alpha}$.

By a similar approximation as in Lemma~\ref{lem:proj} that allows us to subtract out constant terms, we write:
\begin{align}
    f & = \sum_{\comp \in \N^L} v_\comp \phi^\comp~, \\
    \phi_l & = \sum_{\alpha \neq 0} c_{l\alpha} \bp_\alpha~.
\end{align}
Note that by Lemma~\ref{lem:partialortho}, $\langle \bp_\alpha \bp_\beta, \phi^\comp \rangle_\A = 0$ unless $|\comp| = 2$.  So similarly to before, we may rewrite
\begin{align*}
    \P_2 f = \sum_{|\comp| = 2} v_\comp \P_2 \phi^\comp~.
\end{align*}

Here we can simplify notation.  Let $\{e_i\}_{i=1}^L$ denote the standard basis vectors in dimension $L$. Every $\comp \in \N^L$ with $|\comp| = 2$ can be written as $\comp = e_i + e_j$, so let $v_{ij} := v_{e_i + e_j}$.  Then we can rewrite:
\begin{align*}
    \P_2 f = \sum_{i \leq j} v_{ij} \P_2 \phi_i \phi_j~.
\end{align*}

Again, by Lemma~\ref{lem:partialortho}, we know $\P_2$ will annihilate any term of the form $\bp_\gamma \bp_\delta$ if it's not the case that $1 \leq |\gamma|, |\delta| \leq N/2$.  One can see this by noting that, for $1 \leq |\alpha|, |\beta| \leq N/2$, then $\{|\alpha|, |\beta|\} \neq \{|\gamma|, |\delta|\}$, and by the Lemma, $\langle \bp_\alpha \bp_\beta, \bp_\gamma \bp_\delta \rangle_\A = 0$.

So we may pass from $\P_2$ to $\P_1$:
\begin{align}
    \P_2 \phi_i \phi_j &= \P_2 \left(\sum_{\gamma \in \N^D} c_{i\gamma} \bp_\gamma \right) \left(\sum_{\delta \in \N^D} c_{j\delta} \bp_\delta \right) \\
    & = \P_2 \sum_{\gamma \in \N^D} \sum_{\delta \in \N^D} c_{i\gamma} c_{j\delta} \bp_\gamma\bp_\delta \\
    & = \sum_{1 \leq |\gamma| \leq N/2} \sum_{1 \leq |\delta|\leq N/2} c_{i\gamma} c_{j\delta} \bp_\gamma\bp_\delta \\
    & = \left(\sum_{1 \leq |\gamma| \leq N/2} c_{i\gamma} \bp_\gamma \right) \left(\sum_{1 \leq |\delta| \leq N/2} c_{j\delta} \bp_\delta \right) \\
    & = (\P_1 \phi_i) (\P_1 \phi_j)~.
\end{align}

So ultimately we get
\begin{align}
    \P_2 f & = \sum_{i\leq j}^L v_{ij} (\P_1 \phi_i) (\P_1 \phi_j)~.
\end{align}
\end{proof}

\subsection{A Diagonal Inner Product}\label{sec:diagonal-inner-product}

Before we can apply Lemma~\ref{lem:rank}, which lets us transform function approximation error into matrix approximation error, we need a better inner product, one that is diagonal in the low-degree multisymmetric powersum basis.

Consider two more inner products, defined for $f,g$ in the range of $\P_2$:
\begin{align}
    \langle f, g \rangle_{\AA} = \mathbb{E}_{y \sim V, q \sim (S^1)^D, r = 0} \left[f(X) \conj{g(X)}\right]~.
\end{align}

This is nearly the same distribution as before, except we fix $r = 0$.

Then define
\begin{align}
    \langle f, g \rangle_{\diff} = \langle f, g \rangle_{\A}  - 2 \langle f, g \rangle_{\AA} ~.
\end{align}

Because $f$ and $g$ are restricted to the range of $\P_2$, we demonstrate positive-definiteness of this object, and therefore it is a valid inner product.
\begin{theorem}\label{thm:bilinear-form}
    The bilinear form $\langle \cdot, \cdot \rangle_{\diff}$ is an inner product when restricted to the range of $\P_2$.  Furthermore, it is diagonal in the powersum basis $p_{\{\alpha,\beta\}}$ for $1 \leq |\alpha|, |\beta| \leq  N/2$.
\end{theorem}
\begin{proof}
    Given $\tbp_\alpha \tbp_\beta, \tbp_\gamma \tbp_\delta \in im(\P_2)$, we can consider $\langle \tbp_\alpha \tbp_\beta, \tbp_\gamma \tbp_\delta \rangle_{\AA}$ which can similarly be calculated via separability:
    \begin{align*}
        \langle \tbp_\alpha \tbp_\beta, \tbp_\gamma \tbp_\delta \rangle_{\AA}
        & = \langle \tp_{|\alpha|} \tp_{|\beta|} , \tp_{|\gamma|} \tp_{|\delta|} \rangle_\V \cdot \langle q^{\alpha + \beta} , q^{\gamma + \delta} \rangle_\S \\
        & = (1 + \ind_{|\alpha| = |\beta|}) \cdot \ind_{\{|\alpha|, |\beta|\} = \{|\gamma|, |\delta|\}} \cdot \ind_{\alpha + \beta = \gamma + \delta}~.
    \end{align*}
    
    It follows from Lemma~\ref{lem:partialortho} that:
    \begin{align*}
        \langle \tbp_\alpha \tbp_\beta, \tbp_\gamma \tbp_\delta \rangle_{\diff} & = \langle \tbp_\alpha \tbp_\beta, \tbp_\gamma \tbp_\delta \rangle_{\A} - 2\langle \tbp_\alpha \tbp_\beta, \tbp_\gamma \tbp_\delta \rangle_{\AA}\\
        & = 2 \cdot (1 + \ind_{|\alpha| = |\beta|}) \cdot (\ind_{(\alpha, \beta) = (\gamma, \delta)} + \ind_{(\alpha, \beta) = (\delta, \gamma)})~.
    \end{align*}
    
    To eliminate the ambiguity of $\bp_\alpha \bp_\beta$ vs. $\bp_\beta \bp_\alpha$, let us define $\bp_{\{\alpha, \beta\}}$ equal to both these terms.  Then we can equivalently write:
    \begin{align*}
        \langle \bp_{\{\alpha, \beta\}}, \bp_{\{\gamma, \delta\}} \rangle_{\diff} & = 2 \cdot (1 + \ind_{|\alpha| = |\beta|}) \cdot (1 + \ind_{\alpha = \beta}) \cdot \ind_{\{\alpha, \beta\} = \{\gamma, \delta\}}~.
    \end{align*}
    
    Evaluating the indicator functions under all cases we can see:
    \begin{align*}
        \langle \bp_\alpha \bp_\beta, \bp_\gamma \bp_\delta \rangle_\diff = \begin{cases} 
        0 & \{\alpha, \beta\} \neq \{\gamma, \delta\} \\
           2 & \{\alpha, \beta\} = \{\gamma, \delta\}, \quad |\alpha| \neq |\beta| \\
           4 & \{\alpha, \beta\} = \{\gamma, \delta\}, \quad |\alpha| = |\beta|, \quad \alpha \neq \beta \\
           8 & \{\alpha, \beta\} = \{\gamma, \delta\}, \quad \alpha = \beta
       \end{cases}
    \end{align*}
    
    Then we've shown that the bilinear form $\langle \cdot, \cdot \rangle_\diff$, treated as a matrix in the basis of all $\bp_{\{\alpha, \beta\}}$, is positive-definite and diagonal.  Since this basis spans the range of $\P_2$, it follows that the bilinear form is an inner product.
\end{proof}

\subsection{Proof of Lower Bound}\label{sec:actual-proof-of-lower-bound}

We first prove a lower bound using a slightly simpler hard function $g$, before updating the argument to the true choice of $g$ further below.
\begin{theorem}\label{thm:high-separation-easy}
    Let $D > 1$. In particular, assume $\min(N/2, D-1) \geq 2$.  Then we have
    \begin{align}
        \max_{\|g\|_\A = 1} \min_{f\in\syml} \|f - g\|_\A^2 \geq \frac{1}{6} - \frac{L}{6 \cdot 2^{\min(N/2, D-1)}}~.
    \end{align}
    So for $L \leq 2^{{\min(N/2, D-1)}-3}$ we have a constant lower bound on the approximation error.

\end{theorem}

\begin{proof}
    Define $T = |\{\alpha \in \N^D : |\alpha| = N/2\}|$ and choose the bad function $g = \frac{1}{\sqrt{12T}} \sum_{|\alpha| = N/2}  \bp_{\{\alpha, \alpha\}}$.  
    
    Observe that although $\langle \cdot, \cdot \rangle_\A$ is not fully orthogonal in the powersum basis, we can nevertheless calculate by Lemma~\ref{lem:partialortho} that for $|\alpha| = |\beta| = N/2$:
    \begin{align}
        \langle \bp_{\{\alpha, \alpha\}}, \bp_{\{\beta, \beta\}} \rangle_\A &= 4 \cdot (\ind_{\alpha + \alpha = \beta + \beta} + \ind_{(\alpha, \alpha) = (\beta, \beta)} + \ind_{(\alpha, \alpha) = (\beta, \beta)}) \\
        & = 12 \cdot \ind_{\alpha = \beta}~.
    \end{align}
    
    Therefore we can confirm that $g$ is normalized:
    \begin{align}
        \|g\|_\A^2 & = \frac{1}{12T} \sum_{|\alpha| = N/2} \sum_{|\beta| = N/2} \langle \bp_{\{\alpha, \alpha\}}, \bp_{\{\beta, \beta\}} \rangle_\A \\
        & = \frac{1}{12T} \sum_{|\alpha| = N/2} \sum_{|\alpha| = N/2}  12 \cdot \ind_{\alpha = \beta}  \\
        & = \frac{1}{T} \sum_{|\alpha| = N/2} 1  \\
        & = 1~.
    \end{align}
    
    Again, we have $\P_2 g = g$.
    Now by Lemma~\ref{lem:projhigh}, we may write:
    \begin{align*}
        \P_2 f & = \sum_{i\leq j}^L v_{ij} (\P_1 \phi_i) (\P_1 \phi_j)~.
    \end{align*}
    
    Finally, note that $\langle \cdot, \cdot \rangle_\diff$ obeys the inner product conditions of Lemma~\ref{lem:rank} on the range of $\P_2$, following from orthogonality and the normalization:
    \begin{align*}
        \langle \bp_\alpha \bp_\beta, \bp_\alpha \bp_\beta \rangle_\diff = \begin{cases} 
           2 & |\alpha| \neq |\beta| \\
           4 & |\alpha| = |\beta|, \quad \alpha \neq \beta \\
           8 & \alpha = \beta
       \end{cases}
    \end{align*}
    
    So we can apply Lemma~\ref{lem:rank} to $\P_2 f, \P_2 g$, and the inner product $\langle \cdot, \cdot \rangle_\diff$.
    Hence, we can derive:
    \begin{align}
        \min_{f\in\syml} \|f - g\|_\A^2 & \overset{(a)}{\geq} \min_{f\in\syml} \|\P_2 f - \P_2 g\|_\A^2 \\
        & \overset{(b)}{\geq} \min_{f\in\syml} \|\P_2 f - \P_2 g\|_\diff^2 \\
        & \overset{(c)}{\geq} \min_{C,V} \frac{1}{2} \|C^T V C - 2 * \frac{1}{\sqrt{12T}} I\|_F^2 \\
        & = \min_{C,V} \frac{1}{6T} \|C^T V C - I\|_F^2~.
    \end{align}
    
    Here, $(a)$ follows from the definition of $\P_2$ as an orthogonal projection with respect to $\langle \cdot, \cdot \rangle_\A$, $(b)$ follows from the fact that $\|\cdot\|_\A^2 \geq \|\cdot\|_\diff^2$, and $(c)$ follows from the application of Lemma~\ref{lem:rank}.
    
    These matrices are elements of $\mathbb{C}^{T \times T}$, but the term $C^TVC$ is constrained to rank $L$.  Hence, as before we calculate:
    \begin{align}
        \min_{f\in\syml} \|f - g\|_\A^2 \geq \frac{T - L}{6T} = \frac{1}{6} - \frac{L}{6T}~.
    \end{align}
    
    Letting $m = \min(N/2, D-1)$ and assuming $m \geq 2$, it is a simple bound to calculate
    $$T = \binom{N/2 + D -1}{N/2} \geq \binom{2m}{m} \approx \frac{4^m}{\sqrt{\pi m}} \geq 2^m~,$$
    and the bound follows.
    
\end{proof}

This theorem demonstrates a hard function $g$ that cannot be efficiently approximated by $f \in \syml$ for $L = poly(N, D)$, but it does not yet evince a separation.  Indeed, observing that $\|g\|_\infty = \frac{1}{\sqrt{12T}} N^2 T = \frac{N^2 \sqrt{T}}{\sqrt{12}}$, $g$ has very large magnitude, and there's no obvious way to easily approximate this function by an efficient network in $\symlt$.

Thus, we consider a more complicated choice for $g$, that allows for the separation:

\begin{theorem}\label{thm:high-separation-hard}
    Let $D > 1$.  Then let $g' = \frac{g}{\|g\|_\A}$ for $g$ as defined in Lemma~\ref{lem:g-prop}, such that $\|g'\|_\A = 1$.
    Then for $L \leq N^{-2} \exp(O(D))$:
    \begin{align}
        \min_{f\in\syml} \|f - g'\|_\A^2 \geq \frac{1}{12}~.
    \end{align}
\end{theorem}
\begin{proof}
    The lower bound follows almost identically as before.
    By Lemma \hyperref[g-four]{\ref*{lem:g-prop}.\ref*{g-four}} we still have that $\P_2 g' = g'$.  So we can write
    \begin{align}
        g &=\sum_{1 \leq |\alpha| \leq N/2} g_\alpha \bp_{\{\alpha,\alpha\}} \\
        g' &=\sum_{1 \leq |\alpha| \leq N/2} \frac{g_\alpha}{\|g\|_\A} \bp_{\{\alpha,\alpha\}} ~.
    \end{align}
    
    Thus, by the same reasoning as Theorem~\ref{thm:high-separation-easy} we recover the lower bound:
    \begin{align}
        \min_{f\in\syml} \|f - g'\|_\A^2 & \geq \min_{f\in\syml} \|\P_2 f  - \P_2 g'\|_\A^2 \\
        & \geq \min_{f\in\syml} \|\P_2 f  - \P_2 g'\|_\diff^2 \\
         & \geq \min_{C,V} \frac{1}{2} \|C^T V C - G'\|_F^2~,
    \end{align}
    where $G'$ is the matrix induced by $g'$ as given in Lemma~\ref{lem:rank}, i.e. the diagonal matrix indexed by $G_{\alpha\alpha}' = \frac{2g_\alpha}{\|g\|_\A}$.
    
    Now, by the partial orthogonality of $\langle \cdot, \cdot \rangle_\A$ noted in Lemma~\ref{lem:partialortho}, we have:
    \begin{align}
        \|g\|_\A^2 & = \sum_{1 \leq |\alpha| \leq N/2} \quad \sum_{1 \leq |\beta| \leq N/2} \langle g_\alpha \bp_{\{\alpha, \alpha\}}, g_\beta \bp_{\{\beta, \beta\}} \rangle_\A \\
        & = \sum_{1 \leq |\alpha| \leq N/2} \quad \sum_{1 \leq |\beta| \leq N/2} g_\alpha \overline{g_\beta} ( 12 \cdot \ind_{\alpha = \beta}) \\
        & = 12 \sum_{1 \leq |\alpha| \leq N/2} |g_\alpha|^2 ~.
    \end{align}
    Hence, we can say
    \begin{align}
        \|G'\|_F^2 & =  \sum_{1 \leq |\alpha| \leq N/2} \left|\frac{2g_\alpha}{\|g\|_\A}\right|^2 \\
        & = \frac{4 \sum_{1 \leq |\alpha| \leq N/2} |g_\alpha|^2 }{12 \sum_{1 \leq |\alpha| \leq N/2} |g_\alpha|^2} \\
        & = \frac{1}{3}~.
    \end{align}
    Call $G_L'$ the best rank-$L$ approximation of $G'$ in the Frobenius norm.  By classical properties of SVD it follows that $G_L'$ is a diagonal matrix with $L$ entries corresponding to the $L$ largest elements of $G'$.  Then because $\|G'\|_F^2 = \frac{1}{3}$:
    \begin{align}
        \|G_L' - G'\|_F^2 = \frac{1}{3} - \sum_{l=1}^L \left(\frac{|2g_{\alpha_l}|}{\|g\|_\A}\right)^2~,
    \end{align}
    where we order $|g_{\alpha_l}|$ in non-increasing order.
    
    Combining Lemma \hyperref[g-two]{\ref*{lem:g-prop}.\ref*{g-two}} and \hyperref[g-four]{\ref*{lem:g-prop}.\ref*{g-four}} yields the inequality that for all $\alpha$ such that $1 \leq |\alpha| \leq N/2$:
    \begin{align}
        \left(\frac{|2g_\alpha|}{\|g\|_\A}\right)^2 \leq 4N^2\left(1 - \left(\frac{1}{4}\right)^2\right)^{2\D}~,
    \end{align}
    so we can conclude
    \begin{align}
        \min_{f\in\syml} \|f - g'\|_\A^2 & \geq \frac{1}{2}\|G_L' - G'\|_F^2 \\
        & \geq \frac{1}{6} - 2LN^2\left(1 - \left(\frac{1}{4}\right)^2\right)^{2\D}~.
    \end{align}
    Hence, if $L \leq \frac{1}{24} \cdot N^{-2} \left(\frac{16}{15} \right)^{2D}$, we derive a lower bound:
    \begin{align}
        \min_{f\in\syml} \|f - g'\|_\A^2 & \geq \frac{1}{12}~.
    \end{align}
\end{proof}

We remark here that in the instance $D > \sqrt{N/2}$, we replace $D$ with $\hat{D}$ in the above bound, which is consistent with Theorem~\ref{thm:main-result}.

\section{Definition of hard function $g$}\label{sec:construct-g}

In this section we incrementally build the (unnormalized) hard function $g$, ultimately for the sake of Lemma~\ref{lem:g-prop}.  This lemma characterizes all the properties of $g$ that we need to guarantee the lower and upper bounds.

\begin{remark}
    In the following section, we assume $D \leq \sqrt{N/2}$ for simplicity of exposition.  In the case that $D > \sqrt{N/2}$, we replace all instances of $D$ in our functional definitions with $\hat{D} = \min(D, \sqrt{N/2})$, which is only necessary for a projection argument in Lemma~\ref{lem:g-prop} and makes no meaningful change to the proofs.
\end{remark}

\subsection{Mobius transform}

We begin with the following, with $\xi \in \mathbb{C}$ and $|\xi| = 1$.  And in the sequel, we always fix $r = 1/4$.
Consider the 1-D Mobius transformation, with its truncated variant with $t\geq 1$:
\begin{align}
    \mu(\xi) &= \frac{\xi - r}{r\xi - 1} \\
    \hat{\mu}_t(\xi) &= \left(1 - (r\xi)^{t}\right)\cdot \mu(\xi) \\
    & = (r - \xi)\cdot\left(1 + r\xi + (r\xi)^2 + \dots + (r\xi)^{t - 1}\right)
\end{align}

\begin{lemma}\label{lem:mu-prop}
    The following properties hold (where infinity norms are defined with respect to $\S$):
    \begin{enumerate}
        \item \label{mu-one} $\|\mu\|_\infty = 1$
        \item \label{mu-two} $\|\mu\|_\S = 1$
        \item \label{mu-three} $\|\hat{\mu}_t\|_\infty \leq 1 + r^t$
        \item \label{mu-four} $\|\hat{\mu}_t\|_\S^2 = 1 + r^{2t}$
        \item \label{mu-five} $\langle\hat{\mu}_t, 1\rangle_\S = r$,  $\langle \hat{\mu}_t, \xi\rangle_\S = r^2 - 1$ and $|\langle \hat{\mu}_t, \xi^a\rangle_\S| < 1-r^2$ for all $a \geq 2$
        \item \label{mu-six} For $|\xi| = 1, |\omega| \leq 1 + \frac{1}{t}$,
        \begin{align}
            |\hat{\mu}_t(\xi) - \hat{\mu}_t(\omega)| \leq 6 |\xi - \omega|
        \end{align}
    \end{enumerate}
\end{lemma}

\begin{proof}
It is a fact~\parencite{garnett2007bounded} that $\mu$ analytically maps the unit disk to itself, and additional the unit circle to itself, i.e. for any $|\xi| = 1$ we have $|\mu(\xi)| = 1$.  Hence $\|\mu\|_\infty = \|\mu\|_\S = 1$.

We can see that truncation gently perturbs this fact, so for $|\xi| = 1$:
\begin{align}
    |\hat{\mu}_t(\xi)| & = |1 - (r\xi)^t| \cdot |\mu(\xi)|\\
    & \leq 1 + r^t
\end{align}
Additionally, we can calculate the coefficient on each monomial in $\hat{\mu}$:
\begin{align}
    \langle \hat{\mu}_t, \xi^a \rangle_\S = \begin{cases} 
      r & a = 0 \\
      -(r^{a-1} - r^{a+1}) & 1 \leq a \leq t - 1 \\
      -r^{t-1} & a = t \\
      0 & a \geq t
   \end{cases}
\end{align}
It is easy to confirm that the value of $|\langle \hat{\mu}_t, \xi^a \rangle_\S|$ is maximized at $a = 1$.
Hence, we can write the $L_2$ norm:
\begin{align}
    \|\hat{\mu}_t\|_{\S}^2 &= \sum_{a=0}^\infty |\langle \hat{\mu}_t, \xi^a \rangle_\S|^2 \\
    & = r^2 + \sum_{a=1}^{t-1} \left(r^{a-1} - r^{a+1}\right)^2 + r^{2t-2} \\
    & = r^2 + \sum_{a=1}^{t-1} \left(r^{2a-2} -2r^{2a} + r^{2a+2}\right) + r^{2t-2} \\
    & = 1 + r^{2t} 
\end{align}
Finally, for $|\xi| = 1, |\omega| \leq 1 + \frac{1}{t} \leq 2$:
\begin{align}
    |\mu(\xi) - \mu(\omega)| & = \left| \frac{\xi - r}{r\xi - 1} - \frac{\omega - r}{r\omega - 1} \right| \\
    & = \left| \frac{(r^2-1)(\xi - \omega)}{(r\xi - 1)(r\omega - 1)}  \right|~.
\end{align}
So noting $r = \frac{1}{4}$ we get
\begin{align}
    |\mu(\xi) - \mu(\omega)| & \leq \frac{8}{3}|\xi - \omega|~.
\end{align}
Thus:
\begin{align}
    |\hat{\mu}(\xi) - \hat{\mu}(\omega)| & = \left|\left(1 - (r\xi)^{t}\right)\cdot \mu(\xi) - \left(1 - (r\omega)^{t}\right)\cdot \mu(\omega) \right| \\
    & \leq \left|\left(1 - (r\xi)^{t}\right)\cdot \mu(\xi) - \left(1 - (r\omega)^{t}\right)\cdot \mu(\xi) \right| + \left|\left(1 - (r\omega)^{t}\right)\cdot \mu(\xi) - \left(1 - (r\omega)^{t}\right)\cdot \mu(\omega) \right| \\
    & \leq |\mu(\xi)| \cdot r^t |\xi^t - \omega^t| + |1 - (r\omega)^t| \cdot |\mu(\xi) - \mu(\omega)|\\
    & \leq r^t |\xi^t - \omega^t| + |1 - (r\omega)^t| \cdot \frac{8}{3} |\xi - \omega|~.
\end{align}
Note that for $|\xi| = 1, |\omega| \leq 1 + \frac{1}{t}$, because $|\omega|^k \leq e$ for $k \leq t$, we have
\begin{align}
    \left|\xi^{t} - \omega^{t} \right| = \left|(\xi - \omega)(\xi^{t-1} + \xi^{t-2}\omega + \dots + \xi \omega^{t-2} + \omega^{t-1} \right| \leq et |\xi - \omega|~.
\end{align}

Further plugging in that $r = \frac{1}{4}$ and $t \geq 1$:
\begin{align}
    |\hat{\mu}(\xi) - \hat{\mu}(\omega)| & \leq 4^{-t} et |\xi - \omega| + \left(1 + 4^{-t}e\right) \cdot \frac{8}{3} |\xi - \omega| \\
    & < 6 |\xi - \omega|~.
\end{align}

\end{proof}

\subsection{$h$ function}
Now, consider $z \in \mathbb{C}^D$ with $|z_i| = 1$ for all $i$.
We now define:
\begin{align}
    h(z) &= \prod_{i=1}^{\D} \hat{\mu}_{\D}(z_i)~.
\end{align}
\begin{lemma}\label{lem:h-prop}
    The following are true:
    \begin{enumerate}
        \item \label{h-one} $\|h\|_\infty \leq 1 + 2^{-\D}$
        \item \label{h-two} $1 \leq \|h\|_\S^2 \leq 1 + 2^{-\D}$
        \item \label{h-three} For $z, z' \in (S^1)^D$
        $$ |h(z) - h(z')| \leq 12 \|z-z'\|_{1}~. $$
    \end{enumerate}
\end{lemma}

\begin{proof}
    We can immediately bound:
    \begin{align}
        \|h\|_\infty &= \prod_{i=1}^{\D} \|\hat{\mu}_{\D}\|_\infty\\
        & \overset{(a)}{\leq} \left(1 + r^{\D}\right)^{\D} \\
        &\overset{(b)}{\leq} 1 + 2^{\D} \cdot r^{\D} \\
        &\leq 1 + 2^{-\D}~,
    \end{align}
    where $(a)$ follows from Lemma \hyperref[mu-three]{\ref*{lem:mu-prop}.\ref*{mu-three}} and $(b)$ follows from the binomial identity that $(1+x)^t \leq 1 + 2^tx$  for $x\in[0,1], t \geq 1$.  In the last line we simply plug in $r = 1/4$.
    
Similarly by Lemma \hyperref[mu-four]{\ref*{lem:mu-prop}.\ref*{mu-four}}, 
    \begin{align}
        \|h\|_\S^2 &= \prod_{i=1}^{\D} \|\hat{\mu}_{\D}\|_\S^2\\
        & = \left(1 + r^{2\D}\right)^{\D}\\
        & \leq \left(1 + r^{\D}\right)^{\D}~.
    \end{align}
    And so by the same binomial inequality, we have
    \begin{align}
        1 \leq  \|h\|_\S^2 \leq 1 + 2^{-\D}~.
    \end{align}
    Finally, observe that:
    \begin{align}
        |h(z) - h(z')| & \leq \sum_{i=1}^{\D} \left| \left(\prod_{j=1}^{i-1} \hat{\mu}_{\D}(z_j)\right) (\hat{\mu}_{\D}(z_i) - \hat{\mu}_{\D}(z_i')) \left(\prod_{j=i+1}^{\D} \hat{\mu}_{\D}(z_j')\right) \right| \\
        & \overset{(a)}{\leq} \sum_{i=1}^{\D} \left|\hat{\mu}_{\D}(z_i) - \hat{\mu}_{\D}(z_i')\right| (1 + r^{\D})^{\D-1} \\
        & \overset{(b)}{\leq} 6 \sum_{i=1}^{\D} |z_i - z_i'| \left(1 + 2^{-\D}\right) \\
        & \leq 12 \|z-z'\|_1~,
    \end{align}
    where in $(a)$ we apply \hyperref[mu-three]{\ref*{lem:mu-prop}.\ref*{mu-three}}, and in $(b)$ we apply \hyperref[mu-six]{\ref*{lem:mu-prop}.\ref*{mu-six}} and the same binomial identity as above.
\end{proof}

\subsection{$g$ function}

Now, reminding $z_{n,n'} = x_n \circ x_{n'}$, let:
\begin{align}
    g(X) &= -4N^2r^{\D} + \sum_{n,n'=1}^{2N} h(z_{n,n'}) ~.
\end{align}

Note that we subtract a constant here to ensure $g$ has no constant term, which will be necessary for the fact $\P_2 g = g$.
\begin{remark}
    The following lemma is the only place we explicitly require the assumption $D \leq \sqrt{N/2}$, as this guarantees that $\P_2 g = g$.  In the case that $D > \sqrt{N/2}$, we simply replace all instances of $D$ in this section with $\hat{D} = \min(D, \sqrt{N/2})$.  This ensures $g$ is only supported on $\bp_{\{\alpha, \alpha\}}$ with $|\alpha| \leq \hat{D}^2 \leq N/2$.  And the subsequent proofs are identical.
\end{remark}
\begin{lemma}\label{lem:g-prop}
    The following are true:
    \begin{enumerate}
        \item \label{g-one} $\|g\|_\infty \leq 12N^2$. 
        \item \label{g-two} $1 \leq \|g\|_\A^2 \leq 3N^2(1+2^{-\D})$.
        \item \label{g-three}$\P_2g = g$.
        \item \label{g-four} We may write $g =\sum_{1 \leq |\alpha| \leq N/2} g_\alpha \bp_{\{\alpha,\alpha\}}$, where $|g_\alpha|^2 \leq N^2 (1-r^2)^{2\D}$.
        \item \label{g-five} $\mathrm{Lip}(g) \leq 48 N \sqrt{ND}$.
    \end{enumerate}
\end{lemma}

\begin{proof}
First, it's easy to see from Lemma \hyperref[h-one]{\ref*{lem:h-prop}.\ref*{h-one}}
\begin{align}
    \|g\|_\infty &\leq |-4N^2r^{\D}| + 4N^2 \|h\|_\infty \\
    & \leq 4N^2 \left(2^{-2\D} + 1 + 2^{-\D}\right) \\
    & \leq 12N^2~.
\end{align}
Let us expand $h$ as 
\begin{align}
    h(z) = \sum_{\|\alpha\|_\infty \leq \D} h_\alpha z^\alpha~,
\end{align}
noting that by definition of $\hat{\mu}_{\D}$ and Lemma \hyperref[mu-five]{\ref*{lem:mu-prop}.\ref*{mu-five}} we have the constant term $h_0 = r^{\D}$.

Now we can expand
\begin{align}
    g(X) & = -4N^2r^{\D} + \sum_{n,n'=1}^{2N} h(z_{n,n'})\\
    & = -4N^2r^{\D} + \sum_{n,n'=1}^{2N} \left[ r^{\D} + \sum_{1 \leq \|\alpha\|_\infty \leq \D} h_\alpha z_{n,n'}^\alpha \right] \\
    & = \sum_{n,n'=1}^{2N} \sum_{1 \leq \|\alpha\|_\infty \leq \D} h_\alpha z_{n,n'}^\alpha \\
    & = \sum_{1 \leq \|\alpha\|_\infty \leq \D} h_\alpha \sum_{n,n'=1}^{2N} \prod_{d=1}^D (x_{dn}x_{dn'})^{\alpha_d} \\
    & = \sum_{1 \leq \|\alpha\|_\infty \leq \D} h_\alpha |\alpha|  \left(\frac{1}{\sqrt{|\alpha|}}\sum_{n=1}^{2N} \prod_{d=1}^D x_{dn}^{\alpha_d} \right)  \left(\frac{1}{\sqrt{|\alpha|}}\sum_{n'=1}^{2N} \prod_{d'=1}^D x_{d'n'}^{\alpha_{d'}} \right) \\
    & = \sum_{1 \leq \|\alpha\|_\infty \leq \D}  h_\alpha |\alpha| \bp_{\{\alpha, \alpha\}}(X) ~.
\end{align}
Note that $\|\alpha\|_\infty \leq \D$ implies $|\alpha| \leq {\D}^2 \leq N/2$, so it clearly follows that $\P_2 g = g$.
So by Lemma~\ref{lem:partialortho}, $\langle \bp_{\{\alpha, \alpha\}}, \bp_{\{\beta, \beta\}} \rangle_\A = 12 \cdot \ind_{\alpha = \beta}$ whenever $1 \leq |\alpha|, |\beta| \leq N/2$, so we can handily calculate:
\begin{align}
    \|g\|_\A^2 & = \sum_{1 \leq \|\alpha\|_\infty \leq D}  h_\alpha^2 |\alpha|^2 \|\bp_{\{\alpha, \alpha\}}\|_\A^2 \\
    & \leq 12 \cdot (N/2)^2 \sum_{1 \leq \|\alpha\|_\infty \leq D} h_\alpha^2 \\
    & \leq 3 N^2 \|h\|_\S^2 \\
    & \leq 3 N^2 \left(1 + 2^{-\D}\right)~,
\end{align}
where the last line uses Lemma \hyperref[h-two]{\ref*{lem:h-prop}.\ref*{h-two}}.

And likewise 
\begin{align}
    \|g\|_\A^2 & = \sum_{1 \leq \|\alpha\|_\infty \leq \D}  h_\alpha^2 |\alpha|^2 \|\bp_{\{\alpha, \alpha\}}\|_\A^2 \\
    & \geq 12 \left(-r^D + \sum_{\|\alpha\|_\infty \leq \D} h_\alpha^2\right) \\
    & = 12 (-r^D + \|h\|_\S^2)\\
    & \geq 1~,
\end{align}
and the last line again uses Lemma \hyperref[h-two]{\ref*{lem:h-prop}.\ref*{h-two}}.
Finally, note that for any $\alpha$ such that $|\alpha| \leq N/2$, applying Lemma \hyperref[mu-five]{\ref*{lem:mu-prop}.\ref*{mu-five}}.
\begin{align}
    |g_\alpha|^2 = \left|h_\alpha |\alpha|\right|^2 & = |\alpha|^2 \prod_{i=1}^{\D} |\langle \hat{\mu}_{\D}, \xi^{\alpha_i} \rangle_\S|^2 \\
    & \leq N^2 (1-r^2)^{2\D}~.
\end{align}
Finally we consider the Lipschitz norm.  For $X, \hat{X} \in \mathbb{C}^{D \times 2N}$ with each entry of unit norm, it's easy to confirm by Lemma \hyperref[h-three]{\ref*{lem:h-prop}.\ref*{h-three}} that:
\begin{align}
    |g(X) - g(\hat{X})| &\leq \sum_{n,n'=1}^{2N} |h(z_{n,n'}) - h(\hat{z}_{n,n'}')|\\
    & \leq 12 \sum_{n,n'=1}^{2N} \|z_{n,n'} - \hat{z}_{n,n'}\|_1 \\
    & = 12 \sum_{n,n'=1}^{2N} \sum_{d=1}^D |x_{dn} x_{dn'} - \hat{x}_{dn} \hat{x}_{dn'}| \\
    & \leq 12 \sum_{n,n'=1}^{2N} \sum_{d=1}^D |x_{dn}| \cdot |x_{dn'} - \hat{x}_{dn'}| + |\hat{x}_{dn'}| \cdot |x_{dn} - \hat{x}_{dn}| \\
    & = 48 N \sum_{n=1}^{2N} \sum_{d=1}^D |x_{dn} - \hat{x}_{dn}| \\
    & = 48 N \|X - \hat{X}\|_1 \\
    & \leq 48 N \sqrt{2ND} \|X - \hat{X}\|_2
\end{align}
\end{proof}

\section{Upper Bound of Main Result}\label{sec:upper-bound}

In this section we prove the upper bound to representing $g$ with an admissible activation that satisfies Assumption~\ref{ass:act-real}.  

The strategy is as follows.  In Section~\ref{sec:exact-rep} we exactly encode the hard function $g$ with an efficient network, but allowing the choice of very particular activation functions.  In Section~\ref{sec:approx-rep}, we leverage Assumption~\ref{ass:act-real} to build a network that approximates the exact one, using a given activation.  We complete the proof in Section~\ref{sec:proof-upper-bound} by showing the exact and approximate networks stay close together, inducting through the layers.

\subsection{Exact Representation}\label{sec:exact-rep}

Let us first describe how to write $g$ exactly with a network in $\symlt$, using particular activations.  We can then demonstrate to approximate those activations, which only introduces a polynomial dependence in the desired error bound $\epsilon$.

For exact representation, the activations we will allow are $\xi \rightarrow \xi^2$, and $\xi \rightarrow \hat{\mu}_{D}(\xi)$.  Note that from the fact that $\xi \cdot \omega = \frac{1}{2} \left((\xi + \omega)^2 - \xi^2 - \omega^2\right)$, we can exactly multiply scalars with these activations.

Then consider the following structure for $f \in \symlt$ with $L = 1$.  Given $x, x' \in \mathbb{C}^D$ with $|x_i| = |x_i'| = 1$ for all $i$, we define $\psi_1^*(x,x')$ via a network as follows.  In particular, we will use $\cdot$ to explicitly indicate all scalar multiplication:
\begin{align}
    z^* &= (x_1 \cdot x_1', \dots, x_D \cdot x_D') \\
    Z^{(1)*} &= \left(\hat{\mu}_{\D}(z_1^*), \dots, \hat{\mu}_{\D}(z_{\D}^*) \right) \in \mathbb{C}^{\D} \\
    Z^{(2)*} &= \left(Z_1^{(1)*} \cdot Z_2^{(1)*}, \dots, Z_{\D-1}^{(1)*} \cdot Z_{\D}^{(1)*} \right) \in \mathbb{C}^{\D/2} \\
    & \dots \\
    Z^{(\log_2 \D)*} &= Z_1^{(\log_2 \D - 1)*} \cdot Z_2^{(\log_2 \D - 1)*} \in \mathbb{C} \\
    \psi_1^*(x,x') &= Z^{(\log_2 \D)*}
\end{align}

In other words, we exactly calculate $\psi_1^*(x,x') = h(x \circ x')$ through $\log_2 \D$ layers by multiplying the terms $\hat{\mu}_{\D}(z_i)$ at each layer.
Note that $|z_i^*| = 1$ for all $i$.  So by applying Lemma \hyperref[mu-three]{\ref*{lem:mu-prop}.\ref*{mu-three}}, it is the case that each entry $|Z_i^{(k)*}| = |\hat{\mu}_D(z_i^*)|^k \leq (1+r^D)^D \leq 1 + 2^{-D}$ for all $k \leq \log_2 D$.

Now, for an input $\xi \in \mathbb{C}$ we define the map
\begin{align}
    \rho^*(\xi) = \frac{-4N^2r^D + \xi}{\|g\|_\A}~,
\end{align}
and it's easy to confirm that we exactly represent:
\begin{align}
    g'(X) = \rho^*\left(\sum_{n,n'=1}^{2N} \psi_1^*(x_n, x_n')\right)~.
\end{align}

\subsection{Approximate Representation}\label{sec:approx-rep}

Now, we can imitate the network above using the exp activation, and control the approximation error in the infinity norm.
Let us assume we've chosen $f_1, f_2$ as in Lemma~\ref{lem:expnet-epsilon}.  Furthermore, let us define $\xi \star \omega = \frac{1}{2} \left(f_1(\xi + \omega) - f_1(\xi) - f_1(\omega) \right)$, so that $\star$ approximates scalar multiplication.

Then we mimic the exact network via:
\begin{align}
    z &= (x_1 \star x_1', \dots, x_D \star x_D') \\
    Z^{(1)} &= \left(f_2(z_1), \dots, f_2(z_D) \right) \in \mathbb{C}^D \\
    Z^{(2)} &= \left(Z_1^{(1)} \star Z_2^{(1)}, \dots, Z_{D-1}^{(1)} \star Z_D^{(1)} \right) \in \mathbb{C}^{D/2} \\
    & \dots \\
    Z^{(\log_2 D)} &= Z_1^{(\log_2 D - 1)} \star Z_2^{(\log_2 D - 1)} \in \mathbb{C} \\
    \psi_1(x,x') &= Z^{(\log_2 D)}~.
\end{align}

In other words, we replace all instances of multiplication $\cdot$ with $\star$, and all instances of $\hat{\mu}_{D}$ with $f_2$.
Finally, we define the map $\rho$ as:
\begin{align}
    \rho(\xi) = \frac{4N^2}{\|g\|_\A} \cdot \left(\frac{\xi}{4N^2} \star 1 - r^D\right)~,
\end{align}
where we can clearly represent the constant $r^D$ via one additional neuron.

\subsection{Proof of Upper Bound}\label{sec:proof-upper-bound}

We complete the approximation of $g'$ by showing the exact and approximate networks are nearly equivalent in infinity norm, leveraging the assumption on our activation.
\begin{theorem}\label{thm:upper-bound-explicit}
    Consider $\epsilon > 0$ such that $\epsilon \leq \min\left(\frac{1}{100}, \frac{1}{12D^2} \right)$.  For $L = 1$, there exists $f \in \symlt$, parameterized with an activation $\sigma$ that satisfies Assumption~\ref{ass:act-real}, with width $O(D^3 + D^2 \log \frac{DN}{\epsilon}$, depth $O(\log D)$, and maximum weight magnitude $D \log D$ such that over inputs $X \in \mathbb{C}^{D \times 2N}$ with unit norm entries:
    \begin{align}
        \|f - g'\|_\infty \leq \epsilon~.
    \end{align}
\end{theorem}

\begin{proof}

Let $f$ be given by the $\symlt$ network calculated in the previous section, i.e.
\begin{align}
    f(X) = \rho\left(\sum_{n,n'=1}^{2N} \psi_1(x_n, x_n') \right)~.
\end{align}

Clearly $L$ = 1.  From Assumption~\ref{ass:act-real} and what it guarantees about $f_1$ and $f_2$, it's clear that the maximum width of $f$ is $O(D^3 + D^2 \log \frac{D}{\epsilon})$, the depth is $O(\log D)$, and the maximum weight magnitude is $O(D \log D)$.

We can prove the quality of approximation by matching layer by layer.
First we note a quick lemma:
\begin{lemma}\label{lem:mult}
    For $|\xi|, |\omega| \leq \frac{3}{2}$:
    \begin{align}
        |\xi \star \omega - \xi \cdot \omega| \leq \frac{3}{2} \epsilon~.
    \end{align}
\end{lemma}
\begin{proof}
    Based on Assumption~\ref{ass:act-real}, note that for $|\xi|, |\omega| \leq \frac{3}{2}$, we have that $|\xi+\omega| \leq 3$ and therefore:
\begin{align}
    |\xi \star \omega - \xi \cdot \omega| & \leq \frac{1}{2}\left( |f_1(\xi + \omega) - (\xi + \omega)^2| + |f_1(\xi) - \xi^2| + |f_1(\omega) - \xi^2|\right) \\
    & \leq \frac{3}{2}\epsilon~.
\end{align}
\end{proof}

It follows that, because all $|x_i| = 1$:
\begin{align}
    \|z^* - z\|_\infty = \max_{i \leq D} |x_i \star x_i' - x_i \cdot x_i'| \leq \frac{3}{2} \epsilon~.
\end{align}
Now, because $|z_i^*| = 1$, it follows from our assumption on $\epsilon$ that $|z_i| \leq 1 + \frac{3}{2}\epsilon \leq 1 + \frac{1}{D}$.
Hence, we can apply Lemma \hyperref[mu-six]{\ref*{lem:mu-prop}.\ref*{mu-six}} and say
\begin{align}
    \|Z^{(1)*} - Z^{(1)}\|_\infty &= \max_{i \leq D} |\hat{\mu}_{D}(z_i^*) - f_2(z_i)| \\
    & \leq \max_{i \leq D} |\hat{\mu}_{D}(z_i^*) - \hat{\mu}_{D}(z_i)| +  |\hat{\mu}_D(z_i) - f_2(z_i)| \\
    & \overset{(a)}{\leq} 6 \left(\frac{3}{2}\epsilon\right) + \epsilon \\
    & \leq 10 \epsilon~.
\end{align}
where $(a)$ follows from Lemma \hyperref[mu-six]{\ref*{lem:mu-prop}.\ref*{mu-six}} and Assumption~\ref{ass:act-real} again.

Note, observe the following inequality, for any $i$:
\begin{align}
    |Z_{2i}^{(1)*} \cdot Z_{2i+1}^{(1)*}  - Z_{2i}^{(1)} \cdot Z_{2i+1}^{(1)} | & \leq |Z_{2i}^{(1)*} \cdot Z_{2i+1}^{(1)*}  - Z_{2i}^{(1)*} \cdot Z_{2i+1}^{(1)} | + |Z_{2i}^{(1)*} \cdot Z_{2i+1}^{(1)}  - Z_{2i}^{(1)} \cdot Z_{2i+1}^{(1)} |\\
    & = |Z_{2i}^{(1)*}| \cdot |Z_{2i+1}^{(1)*} - Z_{2i+1}^{(1)}| + |Z_{2i+1}^{(1)}| \cdot |Z_{2i}^{(1)*} - Z_{2i}^{(1)}| \\
    & = |\hat{\mu}_D(z_{2i}^*)| \cdot 10\epsilon + |f_2(z_{2i+1})| \cdot 10\epsilon \\
    & \overset{(a)}{\leq} 10\epsilon(|\hat{\mu}_D(z_{2i}^*)| + |\hat{\mu}_D(z_{2i+1})| + \epsilon) \\
    & \overset{(b)}{\leq} 10\epsilon\left(|\hat{\mu}_D(z_{2i}^*)| + |\hat{\mu}_D(z_{2i+1}^*)| + 6\left(\frac{3}{2}\epsilon\right) + \epsilon\right) \\
    & \overset{(c)}{\leq} 10\epsilon(1 + r^{D} + 1 + r^{D} + 4\epsilon + \epsilon) \\
    & \overset{(d)}{\leq} 10\epsilon(5/2) \\
    & \leq 25\epsilon~,
\end{align}
where $(a)$ follows from Lemma~\ref{lem:expnet-epsilon}, $(b)$ follows from Lemma \hyperref[mu-six]{\ref*{lem:mu-prop}.\ref*{mu-six}}, $(c)$ follows from Lemma \hyperref[mu-three]{\ref*{lem:mu-prop}.\ref*{mu-three}}, and $(d)$ follows from the fact that $\epsilon \leq \frac{1}{100}$.

Hence, to draw error bounds one layer higher, we calculate:
\begin{align}
    \|Z^{(2)*} - Z^{(2)}\|_\infty &= \max_{i \leq D/2} |Z_{2i}^{(1)*} \cdot Z_{2i+1}^{(1)*}  - Z_{2i}^{(1)} \star Z_{2i+1}^{(1)} | \\
    &\leq \max_{i \leq D/2} |Z_{2i}^{(1)*} \cdot Z_{2i+1}^{(1)*}  - Z_{2i}^{(1)} \cdot Z_{2i+1}^{(1)} | + |Z_{2i}^{(1)} \cdot Z_{2i+1}^{(1)}  - Z_{2i}^{(1)} \star Z_{2i+1}^{(1)} | \\
    & \overset{(a)}{\leq} 25\epsilon + \frac{3}{2} \epsilon \\
    & \leq 27 \epsilon~,
\end{align}
where in line (a) we apply Lemma~\ref{lem:mult} under the assumption that $|Z_i^{(1)}| \leq \frac{3}{2}$ for all $i$.

Note that from Lemma \hyperref[mu-three]{\ref*{lem:mu-prop}.\ref*{mu-three}}
\begin{align}
    |Z_i^{(1)}| &\leq |Z_i^{(1)} - Z_i^{(1)*}| + |Z_i^{(1)*}| \\
    & \leq 10\epsilon + 1 + r^D < \frac{3}{2}
\end{align}
so this assumption is guaranteed.

We induct upwards through layers: assume that $\|Z^{(k)*} - Z^{(k)}\|_\infty \leq 3^{k+1} \epsilon$ for $k \geq 2$.
Then:
\begin{align}
    |Z_{2i}^{(k)*} \cdot Z_{2i+1}^{(k)*}  - Z_{2i}^{(k)} \cdot Z_{2i+1}^{(k)} | & \leq |Z_{2i}^{(k)*} \cdot Z_{2i+1}^{(k)*}  - Z_{2i}^{(k)*} \cdot Z_{2i+1}^{(k)} | + |Z_{2i}^{(k)*} \cdot Z_{2i+1}^{(k)}  - Z_{2i}^{(k)} \cdot Z_{2i+1}^{(k)} |\\
    & = |Z_{2i}^{(k)*}| \cdot |Z_{2i+1}^{(k)*} - Z_{2i+1}^{(k)}| + |Z_{2i+1}^{(k)}| \cdot |Z_{2i}^{(k)*} - Z_{2i}^{(k)}| \\
    & \overset{(a)}{\leq} 3^{k+1}\epsilon(|Z_{2i}^{(k)*}| + |Z_{2i+1}^{(k)}|) \\
    & \overset{(b)}{\leq} 3^{k+1}\epsilon(|Z_{2i}^{(k)*}| + |Z_{2i+1}^{(k)*}| + 3^{k+1}\epsilon) \\
    & \overset{(c)}{\leq} 3^{k+1}\epsilon((1+r^D)^D + (1+r^D)^D + 3^{k+1}\epsilon) \\
    & \overset{(d)}{\leq} 3^{k+1}\epsilon\left(1 + 2^{-D} + 1 + 2^{-D} + \frac{1}{4}\right) \\
    & \leq 3^{k+1}\epsilon \left(\frac{11}{4}\right)~,
\end{align}
where $(a)$ and $(b)$ are both applications of the inductive hypothesis, $(c)$ follows from Lemma \hyperref[mu-three]{\ref*{lem:mu-prop}.\ref*{mu-three}}, $(d)$ is the binomial inequality and the fact that for any $k \leq \log_2 D$:
\begin{align}
    3^{k+1} \epsilon & \leq 3\left(4^{\log_2 D}\right)\epsilon \\
    & = \frac{\epsilon}{3D^2}\\
    & \leq \frac{1}{4}~.
\end{align}

And as before:
\begin{align}
    \|Z^{(k+1)*} - Z^{(k+1)}\|_\infty &= \max_{i} |Z_{2i}^{(k)*} \cdot Z_{2i+1}^{(k)*}  - Z_{2i}^{(k)} \star Z_{2i+1}^{(k)} | \\
    &\leq \max_{i} |Z_{2i}^{(k)*} \cdot Z_{2i+1}^{(k)*}  - Z_{2i}^{(k)} \cdot Z_{2i+1}^{(k)} | + |Z_{2i}^{(k)} \cdot Z_{2i+1}^{(k)}  - Z_{2i}^{(k)} \star Z_{2i+1}^{(k)} | \\
    & \overset{(a)}{\leq} 3^{k+1}\epsilon \left(\frac{11}{4}\right) + \frac{3}{2} \epsilon \\
    & \leq 3^{k+2}\epsilon~,
\end{align}
where in line (a) we apply Lemma~\ref{lem:mult} under the assumption that $|Z_i^{(k)}| \leq \frac{3}{2}$ for all $i$.

Note that as before
\begin{align}
    |Z_i^{(k)}| &\leq |Z_i^{(k)} - Z_i^{(k)*}| + |Z_i^{(k)*}| \\
    & \leq 3^{k+1}\epsilon + (1 + r^D)^D  \\
    & \leq 3^{k+1}\epsilon + 1 + 2^{-D} \leq \frac{3}{2}~,
\end{align}
so the assumption is granted.

Thus, completing the induction and remembering the definition of $\psi_1$, we conclude:
\begin{align}
    \|\psi_1^*(x_n,x_{n'}) - \psi_1(x_n,x_{n'})\|_\infty \leq 3^{\log_2 D + 1}\epsilon < 3D^2 \epsilon~.
\end{align}

Hence, we can finally bound the final networks:
\begin{align}
    \|g' - f\|_\infty &=
    \left\|\rho^*\left(\sum_{n,n'=1}^{2N}\psi_1^*(x_n,x_{n'})\right) - \rho\left(\sum_{n,n'=1}^{2N}\psi_1(x_n,x_{n'})\right)\right\|_\infty \\
    &= \frac{1}{\|g\|_\A} \left\|\sum_{n,n'=1}^{2N}\psi_1^*(x_n,x_{n'}) - 4N^2 \left(\left[\frac{1}{4N^2}\sum_{n,n'=1}^{2N}\psi_1(x_n,x_{n'})\right] \star 1\right)\right\|_\infty \\
    & \overset{(a)}{\leq} 4N^2 \left\|\frac{1}{4N^2} \sum_{n,n'=1}^{2N}\psi_1^*(x_n,x_{n'}) - \left(\left[\frac{1}{4N^2}\sum_{n,n'=1}^{2N}\psi_1(x_n,x_{n'})\right] \star 1\right)\right\|_\infty \\
    & \overset{(b)}{\leq} 4N^2 \left\|\frac{1}{4N^2} \sum_{n,n'=1}^{2N}\psi_1^*(x_n,x_{n'}) - \frac{1}{4N^2} \sum_{n,n'=1}^{2N}\psi_1^*(x_n,x_{n'})\right\|_\infty + 4N^2 \cdot \frac{3}{2} \epsilon \\
    & \leq 4N^2 \left\|\psi_1^*(x,x') - \psi(x,x')\right\|_\infty + 4N^2 \cdot \frac{3}{2} \epsilon \\
    & \leq 12N^2D^2\epsilon + 6N^2\epsilon \\
    & \leq 18N^2D^2 \epsilon~,
\end{align}
where in $(a)$ we apply the lower bound $\|g\|_A \geq 1$ from \hyperref[g-two]{\ref*{lem:g-prop}.\ref*{g-two}} and in $(b)$ we once again apply Lemma~\ref{lem:mult}, valid from the fact that for all $X$ with unit norm entries:
\begin{align}
    \left|\frac{1}{4N^2}\sum_{n,n'=1}^{2N}\psi_1(x_n,x_{n'})\right| \leq 3D^2 \epsilon \leq \frac{3}{2}~.
\end{align}
So it remains to map $\epsilon \rightarrow \frac{\epsilon}{18N^2D^2}$ in order to yield that $\|f - g'\| \leq \epsilon$.  Note that this remapping only changes the maximum width to be $O(D^3 + D^2 \log \frac{ND}{\epsilon}$.
\end{proof}
\section{Activation Assumption for $\exp$}

We prove that the activation $\exp$ satisfies Assumption~\ref{ass:act-real}.

We need the following standard fact, whose proof we include for completeness:

\begin{lemma}\label{lem:unity}
    Fix $J$ and let $\gamma$ be a primitive $J$th root of unity.  Then
    \begin{align}
    \frac{1}{J}\sum_{j=0}^{J-1} \gamma^{ij} = 
        \begin{cases} 
      1 & i \equiv 0 \mod J \\
      0 & i \not\equiv 0 \mod J
  \end{cases}
    \end{align}
\end{lemma}
\begin{proof}
    If $i \equiv 0 \mod J$, then $\gamma^{ij} = 1$ for all integer $j$ and clearly
\begin{align}
    \frac{1}{J}\sum_{j=0}^{J-1} \gamma^{ij} = 1~.
\end{align}
    Suppose $i \not\equiv 0 \mod J$.  Note that any $J$th root of unity $x$ must satisfy $x^J = 1$, or equivalently
    \begin{align}
        (1-x)\left(\sum_{j=0}^{J-1} x^j \right) = 0~.
    \end{align}
    Because $i \not\equiv 0 \mod J$ and $\gamma$ is a primitive root, it follows $\gamma^i \neq 1$ is another root.  Therefore setting $x = \gamma^i$ and factoring out the non-zero term $(1 - \gamma^i)$ gives 
    \begin{align}
        \sum_{j=0}^{J-1} \gamma^{ij} = 0~.
    \end{align}
\end{proof}

Using this fact, we can approximate simple analytic functions via shallow networks in the $\exp$ activation.
\begin{lemma}\label{lem:expnet}
     For any $J \in \N$ with $J > D$, there exists a shallow neural networks $f_1, f_2$ using the $\exp$ activation,  with $O(JD)$ neurons and $O(D\log D)$ weights, such that
    \begin{align}
        \sup_{|\xi| \leq 3}\left|f_1(\xi) - \xi^2\right| &\leq \frac{4}{J!} \left(\frac{3}{4}\right)^J \\
        \sup_{|\xi| \leq 3}\left|f_2(\xi) - \hat{\mu}_{D}(\xi)\right| &\leq 17D \left(\frac{3}{4}\right)^J ~.
    \end{align}
\end{lemma}

\begin{proof}

    Let $\gamma$ be a primitive $J$th root of unity, $r = 1/4$, and let $k \in \N$ such that $0 \leq k \leq J-1$.
    By applying Lemma~\ref{lem:unity} we can define a network $f^{(k)}$ and expand as:
    \begin{align}
        f^{(k)}(\xi) & := \sum_{j=0}^{J-1} \frac{\gamma^{-kj}}{J} \exp(\gamma^j r \xi) \\
        & = \sum_{j=0}^{J-1} \frac{\gamma^{-kj}}{J} \sum_{i=0}^\infty \frac{(\gamma^j r \xi)^i}{i!} \\
        & = \sum_{i=0}^\infty \frac{(r\xi)^i}{i!} \left[\frac{1}{J} \sum_{j=0}^{J-1} \gamma^{(i-k)j}\right] \\
        & = \sum_{i=0}^\infty \frac{(r\xi)^i}{i!} \ind_{i \equiv k \mod J} \\
        & = \sum_{i=0}^\infty \frac{(r\xi)^{iJ + k}}{(iJ + k)!} \\
        & = \frac{(r\xi)^k}{k!} + \sum_{i=1}^\infty \frac{(r\xi)^{iJ + k}}{(iJ + k)!} ~.
    \end{align}
    
    It follows that we can bound:
    \begin{align}
        \sup_{|\xi| \leq 3} \left|f^{(k)}(\xi) - \frac{(r\xi)^k}{k!}\right| & \leq \sum_{i=1}^\infty \left|\frac{(r\xi)^{iJ + k}}{(iJ + k)!}\right| \\
        & \leq \frac{1}{J!} \sum_{i=1}^\infty \left(\frac{3}{4}\right)^{iJ + k} \\
        & \leq \frac{1}{J!} \left(\frac{3}{4}\right)^J \sum_{i=0}^\infty \left(\frac{3}{4}\right)^{iJ} \\
        &  \leq \frac{1}{J!} \left(\frac{3}{4}\right)^J \frac{1}{1 - (3/4)^J} \\
        & \leq \frac{4}{J!} \left(\frac{3}{4}\right)^J~,
    \end{align}
    so we can define
    \begin{align}
        f_1(\xi) := \frac{2}{r^2} f^{(2)}(\xi)
    \end{align}
    with only $J$ neurons each of width magnitude at most $O(1)$, and instantly gain the bound
    \begin{align}
        \sup_{|\xi| \leq 3} \left|f_1(\xi) - \xi^2 \right| &= \sup_{|\xi| \leq 3} \frac{2}{r^2} \left|f^{(k)}(\xi) - \frac{(r\xi)^2}{2!}\right| \\
        & \leq \frac{2}{r^2} \cdot \frac{4}{J!} \left(\frac{3}{4}\right)^J~.
    \end{align}
    Second, we define
    \begin{align}
        f_2(\xi) &:= r \left(\sum_{k=0}^{D-1} k! f^{(k)}(\xi)
        \right) - \sum_{k=1}^{D} \frac{k!}{r} f^{(k)}(\xi)~.
    \end{align}
    First, let us note that, in spite of seeming to have factorial weights, we can write this network with small weights via properties of the exponential:
    \begin{align}
        f_2(\xi) &= r \left(\sum_{k=0}^{D-1} \exp(\log k!) f^{(k)}(\xi)
        \right) - \sum_{k=1}^{D} \frac{\exp(\log k!)}{r} f^{(k)}(\xi)\\
        & = r \sum_{k=0}^{D-1} \sum_{j=0}^{J-1} \frac{\gamma^{-kj}}{J} \exp(\log k! + \gamma^j r \xi)  -  \sum_{k=1}^{D} \frac{1}{r} \sum_{j=0}^{J-1} \frac{\gamma^{-kj}}{J} \exp(\log k! + \gamma^j r \xi)~.
    \end{align}
    
    The network contains $2DJ$ neurons, with the norm of each weight bounded by $O(D \log D)$.
    
    Then using the decomposition $$\hat{\mu}_D(\xi) = r \sum_{k=0}^{D-1} (r\xi)^k - \frac{1}{r} \sum_{k=1}^{D} (r\xi)^k$$  we derive:
    \begin{align}
        \sup_{|\xi| \leq 3} \left|f_2(\xi) - \hat{\mu}_{D}(\xi) \right| & \leq \sup_{|\xi| \leq 3} \left|r \left(\sum_{k=0}^{D-1} k! f^{(k)}(\xi)
        \right) - r \sum_{k=0}^{D-1} (r\xi)^k \right| + \left| \left(\sum_{k=1}^{D} \frac{k!}{r} f^{(k)}(\xi)
        \right) - \frac{1}{r} \sum_{k=1}^{D} (r\xi)^k \right| \\
        & \leq \left(\sum_{k=0}^{D-1} rk!\right) \frac{4}{J!} \left(\frac{3}{4}\right)^J  + \left(\sum_{k=1}^D \frac{k!}{r}\right) \frac{4}{J!} \left(\frac{3}{4}\right)^J \\
        & \leq 17D \left(\frac{3}{4}\right)^J~.
    \end{align}
\end{proof}

Now, let us restate this result, choosing the error rate $\epsilon$ explicitly:

\begin{lemma}\label{lem:expnet-epsilon}
    For any $\epsilon > 0$, there exists a shallow neural networks $f_1, f_2$ using the $\exp$ activation,  with $O\left(D^2 + D \log \frac{D}{\epsilon}\right)$ neurons and $O(D \log D)$ weights, such that
    \begin{align}
        \sup_{|\xi| \leq 3}\left|f_1(\xi) - \xi^2\right| &\leq \epsilon ~,\\
        \sup_{|\xi| \leq 3}\left|f_2(\xi) - \hat{\mu}_{D}(\xi)\right| &\leq \epsilon ~.
    \end{align}
\end{lemma}

We remark again that, in the event $D > \sqrt{N/2}$, we replace $D$ with $\hat{D}$ in order to approximate the Blaschke product $\hat{\mu}_{\hat{D}}$ as this is the function we use to build the hard function $g$ in that case.  So we recover the statement of Assumption~\ref{ass:act-real}.

\end{document}